\documentclass[a4paper,12pt,reqno]{amsart}

\usepackage{graphicx,amssymb,datetime,float,MnSymbol,currfile,tikz}
\usepackage[round]{natbib}
\usetikzlibrary{arrows}
\usetikzlibrary{decorations.markings}

\newtheorem{theorem}{Theorem}
\newtheorem{lemma}[theorem]{Lemma}
\newtheorem{corollary}[theorem]{Corollary}

\def\ci{\!\perp\!}

\def\ra{\rightarrow}
\def\dra{\dashedrightarrow}
\def\la{\leftarrow}
\def\dla{\dashedleftarrow}

\def\to{\multimap}
\def\no{\multimap}

\newcommand{\comments}[1]{}

\tikzset{tt/.style={decoration={
  markings,
  mark=at position .485 with {\arrow{>}},
  mark=at position .515 with {\arrow{<}}},postaction={decorate}}}

\begin{document}

\title[]{Unifying Gaussian LWF and AMP Chain Graphs to Model Interference}

\author[]{Jose M. Pe\~{n}a\\
IDA, Link\"oping University, Sweden\\
jose.m.pena@liu.se}


\maketitle

\begin{abstract}
An intervention may have an effect on units other than those to which it was administered. This phenomenon is called interference and it usually goes unmodeled. In this paper, we propose to combine Lauritzen-Wermuth-Frydenberg and Andersson-Madigan-Perlman chain graphs to create a new class of causal models that can represent both interference and non-interference relationships for Gaussian distributions. Specifically, we define the new class of models, introduce global and local and pairwise Markov properties for them, and prove their equivalence. We also propose an algorithm for maximum likelihood parameter estimation for the new models, and report experimental results. Finally, we show how to compute the effects of interventions in the new models.
\end{abstract}

\section{Motivation}\label{sec:motivation}

Graphical models are among the most studied and used formalisms for causal inference. Some of the reasons of their success are that they make explicit and testable causal assumptions, and there exist algorithms for causal effect identification, counterfactual reasoning, mediation analysis, and model identification from non-experimental data \citep{Pearl2009,Petersetal.2017,Spirtesetal.2000}. However, in causal inference in general and graphical models in particular, one assumes more often than not that there is no interference, i.e. an intervention has no effect on units other than those to which the intervention was administered \citep{VanderWeeleetal.2012}. This may be unrealistic in some domains. For instance, vaccinating the mother against a disease may have a protective effect on her child, and vice versa. A notable exception is the work by \citet{OgburnandVanderWeele2014}, who distinguish three causal mechanisms that may give rise to interference and show how to model them with directed and acyclic graphs (DAGs). In this paper, we focus on what the authors call interference by contagion: One individual's treatment does not affect another individual's outcome directly but via the first individual's outcome. \citet{OgburnandVanderWeele2014} also argue that interference by contagion typically involves feedback among different individuals' outcomes over time and, thus, it may be modeled by a DAG over the random variables of interest instantiated at different time points. Sometimes, however, the variables are observed at one time point only, e.g. at the end of a season. The observed variables may then be modeled by a Lauritzen-Wermuth-Frydenberg chain graph, or LWF CG for short \citep{Frydenberg1990b,Lauritzen1996}: Directed edges represent direct causal effects, and undirected edges represents causal effects due to interference. Some notable papers on LWF CGs for modeling interference by contagion are \citet{Ogburnetal.2018}, \citet{Shpitser2015},\footnote{\citet{Shpitser2015} actually introduces segregated graphs, an extension of LWF CGs with bidirected edges to allow for confounding. Since we do not consider confounding in this paper, segregated graphs reduce to LWF CGs.} \citet{Shpitseretal.2017}, and \citet{Tchetgenetal.2017}. For instance, the previous mother-child example may be modeled with the LWF CG in Figure \ref{fig:models} (a), where $V_1$ and $V_2$ represent the dose of the vaccine administered to the mother and the child, and $D_1$ and $D_2$ represent the severity of the disease. This paper only deals with Gaussian distributions, which implies that the relations between the random variables are linear. Note that the edge $D_1 - D_2$ represents a symmetric relationship but has some features of a causal relationship, in the sense that a change in the severity of the disease for the mother causes a change in severity for the child and vice versa. This seems to suggest that we can interpret the undirected edges in LWF CGs as feedback loops. That is, that every undirected edge can be replaced by two directed edges in opposite directions. However, this is not correct in general, as explained at length by \citet{LauritzenandRichardson2002}.

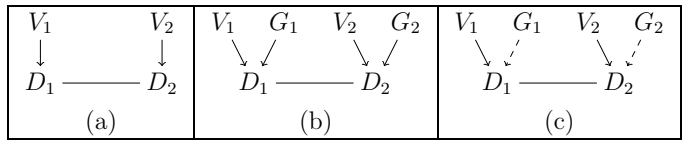
\begin{figure}[t]
\begin{center}
\scalebox{0.8}{
\begin{tabular}{|c|c|c|}
\hline
\begin{tikzpicture}[inner sep=1mm]
\node at (0,0) (V1) {$V_1$};
\node at (0,-1) (H1) {$D_1$};
\node at (2,0) (V2) {$V_2$};
\node at (2,-1) (H2) {$D_2$};
\path[->] (V1) edge (H1);
\path[->] (V2) edge (H2);
\path[-] (H1) edge (H2);
\end{tikzpicture}
&
\begin{tikzpicture}[inner sep=1mm]
\node at (0,0) (V1) {$V_1$};
\node at (1,0) (G1) {$G_1$};
\node at (0.5,-1) (H1) {$D_1$};
\node at (2,0) (V2) {$V_2$};
\node at (3,0) (G2) {$G_2$};
\node at (2.5,-1) (H2) {$D_2$};
\path[->] (V1) edge (H1);
\path[->] (G1) edge (H1);
\path[->] (V2) edge (H2);
\path[->] (G2) edge (H2);
\path[-] (H1) edge (H2);
\end{tikzpicture}
&
\begin{tikzpicture}[inner sep=1mm]
\node at (0,0) (V1) {$V_1$};
\node at (1,0) (G1) {$G_1$};
\node at (0.5,-1) (H1) {$D_1$};
\node at (2,0) (V2) {$V_2$};
\node at (3,0) (G2) {$G_2$};
\node at (2.5,-1) (H2) {$D_2$};
\path[->] (V1) edge (H1);
\path[->,dashed] (G1) edge (H1);
\path[->] (V2) edge (H2);
\path[->,dashed] (G2) edge (H2);
\path[-] (H1) edge (H2);
\end{tikzpicture}
\\
(a) & (b) & (c)\\
\hline
\end{tabular}}
\end{center}\caption{Models of the mother-child example: (a, b) LWF CGs, and (c) UCG.}\label{fig:models}
\end{figure}

In some domains, we may need to model both interference and non-interference. This is the problem that we address in this paper. In our previous example, for instance, the mother and the child may have a gene that makes them healthy carriers, i.e. the higher the expression level of the gene the less the severity of the disease but the load of the disease agent (e.g., a virus) remains unaltered and, thus, so does the risk of infecting the other. We may model this situation with the LWF CG in Figure \ref{fig:models} (b), where $G_1$ and $G_2$ represent the expression level of the healthy carrier gene in the mother and the child. However, this model is not correct: In reality, the mother's healthy carrier gene protects her but has no protective effect on the child, and vice versa. This non-interference relation is not represented by the model. In other words, LWF CGs are not expressive enough to model both interference relations (e.g., intervening on $V_1$ must have an effect on $D_2$) and non-interference relations (e.g., intervening on $G_1$ must have no effect on $D_2$). To remedy this problem, we propose to combine LWF CGs with Andersson-Madigan-Perlman chain graphs, or AMP CGs for short \citep{Anderssonetal.2001}. We call these new models unified chain graphs (UCGs). As we will show, it is possible to describe how UCGs exactly model interference in the Gaussian case. Works such as \citet{Ogburnetal.2018}, \citet{Shpitser2015}, \citet{Shpitseretal.2017}, and \citet{Tchetgenetal.2017} use LWF CGs to model interference and compute some causal effect of interest. However, they do not describe how interference is exactly modeled.

The rest of the paper is organized as follows. Section \ref{sec:preliminaries} introduces some notation and definitions. Section \ref{sec:ucgs} defines UCGs formally. Sections \ref{sec:global} and \ref{sec:local} define global, local and pairwise Markov properties for UCGs and prove their equivalence. Section \ref{sec:MLEs} proposes an algorithm for maximum likelihood parameter estimation for UCGs, and reports experimental results. Section \ref{sec:docalculus} considers causal inference in UCGs. Section \ref{sec:identifiability} discusses identifiability of LWF and AMP CGs. Section \ref{sec:discussion} closes the paper with some discussion. The formal proofs of all the results are contained in Appendix A.

\section{Preliminaries}\label{sec:preliminaries}

In this paper, set operations like union, intersection and difference have the same precedence. When several of them appear in an expression, they are evaluated left to right unless parentheses are used to indicate a different order. Unless otherwise stated, the graphs in this paper are defined over a finite set of nodes $V$. Each node represents a random variable. We assume that the random variables are jointly normally distributed. The graphs contain at most one edge between any pair of nodes. The edge may be undirected or directed. We consider two types of directed edges: Solid ($\ra$) and dashed ($\dra$).

The parents of a set of nodes $X$ in a graph $G$ is the set $Pa(X) = \{A | A \ra B$ or $A \dra B$ is in $G$ with $B \in X \}$. The neighbors of $X$ is the set $Ne(X) = \{A | A - B$ is in $G$ with $B \in X \}$. The adjacents of $X$ is the set $Ad(X) = \{A | A \to B$ or $B \to A$ is in $G$ with $B \in X \}$ where $\to$ means $\ra$, $\dra$ or $-$. The non-descendants of $X$ are $Nd(X) = \{B | $ neither $A \to \cdots \ra \cdots \to B$ nor $A \to \cdots \dra \cdots \to B$ is in $G$ with $A \in X \}$. Moreover, the subgraph of $G$ induced by $X$ is denoted by $G_X$. A route between a node $V_{1}$ and a node $V_{n}$ in $G$ is a sequence of (not necessarily distinct) nodes $V_{1}, \ldots, V_{n}$ such that $V_i \in Ad(V_{i+1})$ for all $1 \leq i < n$. Moreover, $V_j, \ldots, V_{j+k}$ and $V_{j+k}, \ldots, V_{j}$ with $1 \leq j \leq n-k$ are called subroutes of the route. The route is called undirected if $V_i - V_{i+1}$ for all $1 \leq i < n$. If $V_n=V_1$, then the route is called a cycle. A cycle is called semidirected if it is of the form $V_1 \ra V_2 \no \cdots \no V_n$ or $V_1 \dra V_2 \no \cdots \no V_n$. A route of distinct nodes is called a path. A chain graph (CG) is a graph with (possibly) directed and undirected edges, and without semidirected cycles. A set of nodes of a CG $G$ is connected if there exists an undirected route in $G$ between every pair of nodes in the set. A chain component of $G$ is a maximal connected set. Note that the chain components of $G$ can be sorted topologically, i.e. for every edge $A \ra B$ or $A \dra B$ in $G$, the component containing $A$ precedes the component containing $B$. The chain components of $G$ are denoted by $Cc(G)$. CGs without dashed directed edges are known as LWF CGs, whereas CGs without solid directed edges are known as AMP CGs. 

We now recall the interpretation of LWF CGs.\footnote{LWF CGs were originally interpreted via the so-called moralization criterion \citep{Frydenberg1990b,Lauritzen1996}. \citet[Lemma 5.1]{Studeny1998} introduced the so-called separation criterion that we use in this paper and proved its equivalence to the moralization criterion.} Given a route $\rho$ in a LWF CG $G$, a section of $\rho$ is a maximal undirected subroute of $\rho$. A section $C_1 - \cdots - C_n$ of $\rho$ is called a collider section if $A \ra C_1 - \cdots - C_n \la B$ is a subroute of $\rho$. We say that $\rho$ is $Z$-open with $Z \subseteq V$ when (i) all the collider sections in $\rho$ have some node in $Z$, and (ii) all the nodes that are outside the collider sections in $\rho$ are outside $Z$. We now recall the interpretation of AMP CGs.\footnote{\citet{Anderssonetal.2001} originally interpreted AMP CGs via the so-called augmentation criterion. \citet[Theorem 4.1]{Levitzetal.2001} introduced the so-called p-separation criterion and proved its equivalence to the augmentation criterion. \citet[Theorem 2]{Penna2016} introduced the route-based criterion that we use in this paper and proved its equivalence to the p-separation criterion.} Given a route $\rho$ in an AMP CG $G$, $C$ is a collider node in $\rho$ if $\rho$ has a subroute $A \dra C \dla B$ or $A \dra C - B$. We say that $\rho$ is $Z$-open with $Z \subseteq V$ when (i) all the collider nodes in $\rho$ are in $Z$, and (ii) all the non-collider nodes in $\rho$ are outside $Z$. Let $X$, $Y$ and $Z$ denote three disjoint subsets of $V$. When there is no $Z$-open route in a LWF or AMP CG $G$ between a node in $X$ and a node in $Y$, we say that $X$ is separated from $Y$ given $Z$ in $G$ and denote it as $X \ci_G Y | Z$. Moreover, we represent by $X \ci_p Y | Z$ that $X$ and $Y$ are conditionally independent given $Z$ in a probability distribution $p$. We say that $p$ satisfies the global Markov property with respect to $G$ when $X \ci_p Y | Z$ for all $X, Y, Z \subseteq V$ such that $X \ci_G Y |Z$. When $X \ci_p Y | Z$ if and only if $X \ci_G Y | Z$, we say that $p$ is faithful to $G$. Two LWF CGs or AMP CGs are Markov equivalent if the set of distributions that satisfy the global Markov property with respect to each CG is the same. Characterizations of Markov equivalence exist. See \citet[Theorem 5.6]{Frydenberg1990b} for LWF CGs, and \citet[Theorem 5]{Anderssonetal.2001} for AMP CGs. The following theorem gives an additional characterization.

\begin{theorem}\label{the:eq}
Two LWF CGs or AMP CGs $G$ and $H$ are Markov equivalent if and only if they represent the same separations.
\end{theorem}

Finally, let $X$, $Y$, $W$ and $Z$ be disjoint subsets of $V$. A probability distribution $p$ that satisfies the following five properties is called graphoid: Symmetry $X \ci_p Y | Z \Rightarrow Y \ci_p X | Z$, decomposition $X \ci_p Y \cup W | Z \Rightarrow X \ci_p Y | Z$, weak union $X \ci_p Y \cup W | Z \Rightarrow X \ci_p Y | Z \cup W$, contraction $X \ci_p Y | Z \cup W \land X \ci_p W | Z \Rightarrow X \ci_p Y \cup W | Z$, and intersection $X \ci_p Y | Z \cup W \land X \ci_p W | Z \cup Y \Rightarrow X \ci_p Y \cup W | Z$. If $p$ also satisfies the following property, then it is called compositional graphoid: Composition $X \ci_p Y | Z \land X \ci_p W | Z \Rightarrow X \ci_p Y \cup W | Z$. It is known that every Gaussian distribution is a compositional graphoid \citep[Lemma 2.1, Proposition 2.1 and Corollary 2.4]{Studeny2005}.

\section{Unified Chain Graphs}\label{sec:ucgs}

In this section, we introduce our causal models to represent both interference and non-interference relationships. Let $p$ denote a Gaussian distribution that satisfies the global Markov property with respect to a LWF or AMP CG $G$. Then,
\begin{equation}\label{eq:recursion1}
p(V)= \prod_{K \in Cc(G)} p(K|Pa(K))
\end{equation}
because $K \ci_G K_1 \cup \cdots \cup K_n \setminus Pa(K) | Pa(K)$ where $K_1, \ldots, K_n$ denote the chain components of $G$ that precede $K$ in an arbitrary topological ordering of the components. Assume without loss of generality that $p$ has zero mean vector. Moreover, let $\Sigma$ and $\Omega$ denote respectively the covariance and precision matrices of $p(K,Pa(K))$. Each matrix is then of dimension $(|K|+|Pa(K)|) \times (|K|+|Pa(K)|)$. \citet[Section 2.3.1]{Bishop2006} shows that the conditional distribution $p(K | Pa(K))$ is Gaussian with covariance matrix $\Lambda_K$ and the mean vector is a linear function of $Pa(K)$ with coefficients $\beta_K$. Then, $\beta_K$ is of dimension $|K| \times |Pa(K)|$, and $\Lambda_K$ is of dimension $|K| \times |K|$. Moreover, $\beta_K$ and $\Lambda_K$ can be expressed in terms of $\Sigma$ and $\Omega$ as follows:
\begin{equation}\label{eq:recursion2}
K | Pa(K) \sim \mathcal{N}(\beta_K Pa(K), \Lambda_K)
\end{equation}
with
\begin{equation}\label{eq:beta}
\beta_K= \Sigma_{K,Pa(K)} \Sigma_{Pa(K),Pa(K)}^{-1} = -\Omega_{K,K}^{-1} \Omega_{K,Pa(K)}
\end{equation}
and
\begin{equation}\label{eq:lambda}
\Lambda_K= \Sigma_{K|Pa(K)} = \Sigma_{K,K} - \Sigma_{K,Pa(K)}\Sigma_{Pa(K),Pa(K)}^{-1}\Sigma_{Pa(K),K} = \Omega_{K,K}^{-1}
\end{equation}
where $\Sigma_{K,Pa(K)}$ denotes the submatrix of $\Sigma$ with rows $K$ and columns $Pa(K)$, and $\Sigma_{Pa(K),Pa(K)}^{-1}$ denotes the inverse of $\Sigma_{Pa(K),Pa(K)}$. Moreover, $(\Lambda_K^{-1})_{i,j}=0$ for all $i,j \in K$ such that $i - j$ is not in $G$, because $i \ci_G j | K \setminus \{i, j\} \cup Pa(K)$ \citep[Proposition 5.2]{Lauritzen1996}.

Furthermore, if $G$ is an AMP CG, then $(\beta_K)_{i,j}=0$ for all $i \in K$ and $j \in Pa(K) \setminus Pa(i)$, because $i \ci_G Pa(K) \setminus Pa(i) | Pa(i)$. Therefore, the directed edges of an AMP CG are suitable for representing non-interference. On the other hand, the directed edges of a LWF CG are suitable for representing interference. To see it, note that if $G$ is a LWF CG, then $\Omega_{i,j}=0$ for all $i \in K$ and $j \in Pa(K) \setminus Pa(i)$, because $i \ci_G j | K \setminus \{i\} \cup Pa(K) \setminus \{j\}$. However, $(\beta_K)_{i,j}$ is not identically zero. For instance, if $G=\{1 \ra 3 - 4 \la 2\}$ and $K=\{3,4\}$, then
\begin{align*}
(\beta_K)_{4,1} &= -(\Omega_{K,K}^{-1})_{4,3} (\Omega_{K,Pa(K)})_{3,1} -(\Omega_{K,K}^{-1})_{4,4} (\Omega_{K,Pa(K)})_{4,1}\\
&= -(\Omega_{K,K}^{-1})_{4,3} (\Omega_{K,Pa(K)})_{3,1}
\end{align*}
because $(\Omega_{K,Pa(K)})_{4,1}=0$ as $1 \ra 4$ is not in $G$. In other words, if there is a path in a LWF CG $G$ from $j$ to $i$ through nodes in $K$, then $(\beta_K)_{i,j}$ is not identically zero. Actually, $(\beta_K)_{i,j}$ can be written as a sum of path weights over all such paths. This follows directly from Equation \ref{eq:beta} and the result by \citet[Theorem 1]{JonesandWest2005}, who show that $(\Omega_{K,K}^{-1})_{i,l}$ can be written as a sum of path weights over all the paths in $G$ between $l$ and $i$ through some (but not necessarily all) nodes in $K$. Specifically,
\begin{equation}\label{eq:jones}
(\Omega_{K,K}^{-1})_{i,l} = \sum_{\rho \in \pi_{i,l}} (-1)^{|\rho|+1} \frac{|(\Omega_{K,K})_{\setminus \rho}|}{|\Omega_{K,K}|} \prod_{n=1}^{|\rho|-1} (\Omega_{K,K})_{\rho_n, \rho_{n+1}}
\end{equation}
where $\pi_{i,l}$ denotes the set of paths in $G$ between $i$ and $l$ through nodes in $K$, $|\rho|$ denotes the number of nodes in a path $\rho$, $\rho_n$ denotes the $n$-th node in $\rho$, and $(\Omega_{K,K})_{\setminus \rho}$ is the matrix with the rows and columns corresponding to the nodes in $\rho$ omitted. Moreover, the determinant of a zero-dimensional matrix is taken to be 1. As a consequence, a LWF CG $G$ does not impose zero restrictions on $\beta_K$, because $K$ is connected by definition of chain component. Previous works on the use of LWF CGs to model interference (e.g., \citep{Ogburnetal.2018,Shpitser2015,Shpitseretal.2017,Tchetgenetal.2017}) focus on developing methods for computing some causal effects of interest, and do not give many details on how interference is really being modeled. In the case of Gaussian LWF CGs, Equations \ref{eq:beta} and \ref{eq:jones} shed some light on this question. For instance, consider extending the mother-child example in Section \ref{sec:motivation} to a group of friends. The vaccine for individual $j$ has a protective effect on individual $j$ developing the disease, which in its turn has a protective effect on her friends, which in its turn has a protective effect on her friends' friends, and so on. Moreover, the protective effect also works in the reverse direction, i.e. the vaccine for the latter has a protective effect on individual $j$. In other words, the protective effect of the vaccine for individual $j$ on individual $i$ is the result of all the paths for the vaccine's effect to reach individual $i$ through the network of friends. This is exactly what Equations \ref{eq:beta} and \ref{eq:jones} tell us (assuming that the neighbors of an individual's disease node in $G$ are her friends' disease nodes). Appendix B gives an alternative decomposition of the covariance in terms of path coefficients, which may give further insight into Equation \ref{eq:jones}.

The discussion above suggests a way to model both interference and non-interference by unifying LWF and AMP CGs, i.e. by allowing $\ra$, $\dra$ and $-$ edges as long as no semidirected cycle exists. We call these new models unified chain graphs (UCGs). The lack of an edge $-$ imposes a zero restriction on the elements of $\Omega_{K,K}$, as in LWF and AMP CGs. The lack of an edge $\dra$ in an UCG imposes a zero restriction on the elements of $\beta_K$, as in AMP CGs. Finally, the lack of an edge $\ra$ imposes a zero restriction on the elements of $\Omega_{K,Pa(K)}$ but not on the elements of $\beta_K$, as in LWF CGs. Therefore, edges $\dra$ may be used to model non-interference, whereas edges $\ra$ may be used to model interference. For instance, the mother-child example in Section \ref{sec:motivation} may be modeled with the UCG in Figure \ref{fig:models} (c).

Every UCG is a CG, but the opposite is not true due to the following requirement. We divide the parents of a set of nodes $X$ in an UCG $G$ into mothers and fathers as follows. The fathers of $X$ are $Fa(X) = \{B | B \ra A$ is in $G$ with $A \in X \}$. The mothers of $X$ are $Mo(X) = \{B | B \dra A$ is in $G$ with $A \in X \}$. We require that $Fa(K) \cap Mo(K) = \emptyset$ for all $K \in Cc(G)$. Therefore, the lack of an edge $\dra$ imposes a zero restriction on the elements of $\beta_K$ corresponding to the mothers, and the lack of an edge $\ra$ imposes a zero restriction on the elements of $\Omega_{K,Pa(K)}$ corresponding to the fathers. In other words, $G$ imposes the following zero restrictions:
\begin{equation}\label{eq:zerom}
(\beta_K)_{i,j}=0 \text{ for all } i \in K \text{ and } j \in Mo(K) \setminus Mo(i),
\end{equation}
\begin{equation}\label{eq:zerof}
\Omega_{i,j}=0 \text{ for all } i \in K \text{ and } j \in Fa(K) \setminus Fa(i), \text{ and}
\end{equation}
\begin{equation}\label{eq:zeron}
(\Lambda_K^{-1})_{i,j}=0 \text{ for all } i \in K \text{ and } j \in K \setminus Ne(i).
\end{equation}

The reason why we require that $Fa(K) \cap Mo(K) = \emptyset$ is to avoid imposing contradictory zero restrictions on $\beta_K$, e.g. the edge $j \ra i$ excludes the edge $j \dra i$ by definition of CG, which implies that $(\beta_K)_{i,j}$ is identically zero, but $j \ra i$ implies the opposite. In other words, without this constraint, UCGs would reduce to AMP CGs. The following lemma formalizes this statement.

\begin{lemma}\label{lem:requirement}
Without the requirement that $Fa(K) \cap Mo(K) = \emptyset$ for all $K \in Cc(G)$, the UCG $G$ imposes the same zero restrictions on $\beta_K$ and $\Lambda_K^{-1}$ as the AMP CG resulting from removing the edges $\ra$ from $G$.
\end{lemma}

The lemma above implies that the UCG and the AMP CG can model the same Gaussian distributions. Lastly, note that the zero restrictions associated with an UCG (i.e., Equations \ref{eq:zerom}-\ref{eq:zeron}) induce a Gaussian distribution by Equation \ref{eq:recursion1} and \citet[Section 2.3.3]{Bishop2006}.

\section{Global Markov Property}\label{sec:global}

In this section, we present a separation criterion for UCGs. Given a route $\rho$ in an UCG, $C$ is a collider node in $\rho$ if $\rho$ has a subroute $A \dra C \dla B$ or $A \dra C - B$ or $A \dra C \la B$. A section of $\rho$ is a maximal undirected subroute of $\rho$. A section $C_1 - \cdots - C_n$ of $\rho$ is called a collider section if $A \ra C_1 - \cdots - C_n \la B$ is a subroute of $\rho$. We say that $\rho$ is $Z$-open with $Z \subseteq V$ when (i) all the collider nodes in $\rho$ are in $Z$, (ii) all the collider sections in $\rho$ have some node in $Z$, and (iii) all the nodes that are outside the collider sections in $\rho$ and are not collider nodes in $\rho$ are outside $Z$. Let $X$, $Y$ and $Z$ denote three disjoint subsets of $V$. When there is no $Z$-open route in $G$ between a node in $X$ and a node in $Y$, we say that $X$ is separated from $Y$ given $Z$ in $G$ and denote it as $X \ci_G Y | Z$. Note that this separation criterion unifies the criteria for LWF and AMP CGs reviewed in Section \ref{sec:ucgs}. Finally, we say that a probability distribution $p$ satisfies the global Markov property with respect to an UCG $G$ if $X \ci_p Y | Z$ for all $X, Y, Z \subseteq V$ such that $X \ci_G Y |Z$.

The next theorem states the equivalence between the global Markov property and the zero restrictions associated with an UCG.

\begin{theorem}\label{the:zg}
A Gaussian distribution $p$ satisfies Equations \ref{eq:recursion1} and \ref{eq:zerom}-\ref{eq:zeron} with respect to an UCG $G$ if and only if it satisfies the global Markov property with respect to $G$.
\end{theorem}

We have mentioned before that \citet{JonesandWest2005} prove that the covariance between two nodes $i$ and $j$ can be written as a sum of path weights over the paths between $i$ and $j$ in a certain undirected graph (recall Equation \ref{eq:jones} for the details). \citet{Wright1921} proves a similar result for DAGs. The following theorem generalizes both results.

\begin{theorem}\label{the:JonesWestWright}
Given a Gaussian distribution that satisfies the global Markov property with respect to an UCG $G$, the covariance between two nodes $i$ and $j$ of $G$ can be written as a sum of products of weights over the edges of the open paths between $i$ and $j$ in $G$.
\end{theorem}

\section{Block-Recursive, Pairwise and Local Markov Properties}\label{sec:local}

In this section, we present block-recursive, pairwise and local Markov properties for UCGs, and prove their equivalence to the global Markov property for Gaussian distributions. Equivalence means that every Gaussian distribution that satisfies any of these properties with respect to an UCG  also satisfies the global Markov property with respect to the UCG, and vice versa. The relevance of these results stems from the fact that checking whether a distribution satisfies the global Markov property can now be performed more efficiently: We do not need to check that every global separation corresponds to a statistical independence in the distribution, it suffices to do so for the local (or pairwise or block-recursive) separations, which are typically considerably fewer. This is the approach taken by most learning algorithms based on hypothesis tests, such as the PC algorithm for DAGs \citep{Spirtesetal.2000}, and its posterior extensions to LWF CGs \citep{Studeny1997b} and AMP CGs \citep{Penna2012}. The results in this section pave the way for developing a similar learning algorithm for UCGs, something that we postpone to a future article.

We say that a probability distribution $p$ satisfies the block-recursive Markov property with respect to $G$ if for any chain component $K \in Cc(G)$ and vertex $i \in K$
\begin{itemize}
\item[(B1)] $i \ci_p Nd(K) \setminus K \setminus Fa(K) \setminus Mo(i) | Fa(K) \cup Mo(i)$,
\item[(B2)] $i \ci_p Nd(K) \setminus K \setminus Fa(i) \setminus Mo(K) | K \setminus \{i\} \cup Fa(i) \cup Mo(K)$ and
\item[(B3)] $i \ci_p j | K \setminus \{i,j\} \cup Pa(K)$ for all $j \in K \setminus \{i\} \setminus Ne(i)$.
\end{itemize}

\begin{theorem}\label{the:zb}
A Gaussian distribution $p$ satisfies Equations \ref{eq:recursion1} and \ref{eq:zerom}-\ref{eq:zeron} with respect to an UCG $G$ if and only if it satisfies the block-recursive Markov property with respect to $G$.
\end{theorem}

We say that a probability distribution $p$ satisfies the pairwise Markov property with respect to an UCG $G$ if for any non-adjacent vertices $i$ and $j$ of $G$ with $i \in K$ and $K \in Cc(G)$
\begin{itemize}
\item[(P1)] $i \ci_p j | Nd(i) \setminus K \setminus \{j\}$ if $j \in Nd(i) \setminus K \setminus Fa(K) \setminus Mo(i)$, and
\item[(P2)] $i \ci_p j | Nd(i) \setminus \{i,j\}$ if $j \in Nd(i) \setminus \{i\} \setminus Fa(i) \setminus Mo(K)$.

\end{itemize}

\begin{theorem}\label{the:pb}
The pairwise and block-recursive Markov properties are equivalent for graphoids.
\end{theorem}

We say that a probability distribution $p$ satisfies the local Markov property with respect to an UCG $G$ if for any vertex $i$ of $G$ with $i \in K$ and $K \in Cc(G)$
\begin{itemize}
\item[(L1)] $i \ci_p Nd(i) \setminus K \setminus Fa(K) \setminus Mo(i) | Fa(K) \cup Mo(i)$ and
\item[(L2)] $i \ci_p Nd(i) \setminus \{i\} \setminus Pa(i) \setminus Ne(i) \setminus Mo(Ne(i)) | Pa(i) \cup Ne(i) \cup Mo(Ne(i))$.
\end{itemize}

\begin{theorem}\label{the:lp}
The local and pairwise Markov properties are equivalent for Gaussian distributions.
\end{theorem}

Theorems \ref{the:zg} and \ref{the:zb}-\ref{the:lp} imply the following corollary, which summarizes our results.

\begin{corollary}
The global, block-recursive, local and pairwise Markov properties are all equivalent for Gaussian distributions.
\end{corollary}

\section{Maximum Likelihood Parameter Estimation}\label{sec:MLEs}

In this section, we introduce a procedure for computing maximum likelihood estimates (MLEs) of the parameters of an UCG $G$, i.e. the MLEs of the non-zero entries of $(\beta_K)_{K,Mo(K)}$, $\Omega_{K,K}$ and $\Omega_{K,Fa(K)}$ for every chain component $K$ of $G$. Specifically, we adapt the procedure proposed by \citet{DrtonandEichler2006} for computing MLEs of the parameters of an AMP CG. The procedure combines generalized least squares and iterative proportional fitting. Suppose that we have access to a data matrix $D$ whose column vectors are the data instances and the rows are indexed by $V$. Moreover, let $D_X$ denote the data over the variables $X \subseteq V$, i.e. the rows of $D$ corresponding to $X$. Note that thanks to Equation \ref{eq:recursion1}, the MLEs of the parameters corresponding to each chain component $K$ of $G$ can be obtained separately. Moreover, recall that $K | Pa(K) \sim \mathcal{N}(\beta_K Pa(K), \Omega_{K,K}^{-1})$. Therefore, we may compute the MLEs of the parameters corresponding to the component $K$ by iterating between computing the MLEs $\hat{\Omega}_{K,K}$ and $\hat{\Omega}_{K,Fa(K)}$ and thus $(\hat{\beta}_K)_{K,Fa(K)}$ while fixing $(\hat{\beta}_K)_{K,Mo(K)}$, and computing the MLEs $(\hat{\beta}_K)_{K,Mo(K)}$ while fixing $(\hat{\beta}_K)_{K,Fa(K)}$. Specifically, we initialize $(\hat{\beta}_K)_{K,Mo(K)}$ to zero and then iterate through the following steps until convergence:
\begin{itemize}
\item Compute $\hat{\Omega}_{K,K}$ and $\hat{\Omega}_{K,Fa(K)}$ from the data $D_{Fa(K)}$ and the residuals $D_K - (\hat{\beta}_K)_{K,Mo(K)} D_{Mo(K)}$. 
\item Compute $(\hat{\beta}_K)_{K,Fa(K)}$ from $\hat{\Omega}_{K,K}$ and $\hat{\Omega}_{K,Fa(K)}$.
\item Compute $(\hat{\beta}_K)_{K,Mo(K)}$ from the data $D_{Mo(K)}$ and the residuals $D_K - (\hat{\beta}_K)_{K,Fa(K)} D_{Fa(K)}$.
\end{itemize}
The first step corresponds to computing the MLEs of the parameters of a LWF CG and, thus, it is solved by running the iterative proportional fitting procedure as indicated in \citet{Lauritzen1996}. This procedure optimizes the likelihood function iteratively over different sections of the parameter space. Specifically, each iteration adjusts the covariance matrix for one clique marginal. The second step above is solved by Equation \ref{eq:beta}. The third step corresponds to computing the MLEs of the parameters of an AMP CG and, thus, it is solved analytically by Equation 13 in \citet{DrtonandEichler2006}. Note that to guarantee convergence to the MLEs, the first and third step should be solved jointly. Therefore, our procedure is expected to converge to a local rather than global maximum of the likelihood function. As \citet{DrtonandEichler2006} note, this is also the case for their procedure, upon which ours builds. As for convergence, note that our procedure consists in interleaving the iterative proportional fitting procedure in the first step, and the analytical solution to generalized least squares in the third step. \citet[Proposition 1]{DrtonandEichler2006} prove that each of these steps increases the likelihood function, which implies convergence since the likelihood function is bounded. See also \citet[Theorem 5.4]{Lauritzen1996}. The experiments in the next section confirm that our procedure converges to satisfactory estimates within a few iterations.

\subsection{Experimental Evaluation}

First, we generate 1000 UCGs as follows. Each UCG consists of five mothers, five fathers and 10 children. The edges are sampled independently and with probability 0.2 each. If the edges sampled do not satisfy the following two constraints, the sampling process is repeated: (i) There must be an edge from every parent to some child, i.e. the parents are real parents, and (ii) the children must be connected by an undirected path, i.e. they conform a chain component.

Then, we parameterize each of the 1000 UCGs generated above as follows. The 10 children are denoted by $K$, and the 10 parents by $Pa(K)$. The non-zero elements of $\beta_K$ corresponding to the mothers are sampled uniformly from the interval [-3,3]. The non-zero elements of $\Omega_{K,K}$ and $\Omega_{K,Fa(K)}$ are sampled uniformly from the interval [-3,3], with the exception of the diagonal elements of $\Omega_{K,K}$ which are sampled uniformly from the interval [0,30]. If the sampled $\Omega_{K,K}$ is not positive definite then the sampling process is repeated. The reason why the diagonal elements are sampled from a wider interval is to avoid having to repeat the sampling process too many times.

Finally, note that each of the 1000 parameterized UCGs generated above specifies one probability distribution $p(K|Pa(K))$. The goal of our experiments is to evaluate the accuracy of the algorithm presented before to estimate the parameter values corresponding to $p(K|Pa(K))$ from a finite sample of $p(K,Pa(K))$. To generate these learning data, we sample first $p(Pa(K))$ and then $p(K|Pa(K))$. Each of the 1000 probability distributions $p(Pa(K))$ is constructed as follows. The off-diagonal elements of $\Omega_{Pa(K),Pa(K)}$ are sampled uniformly from the interval [-3,3], whereas the diagonal elements are sampled uniformly from the interval [0,30]. As before, we repeat the sampling process if the resulting matrix is not positive definite. Note that no element of $\Omega_{Pa(K),Pa(K)}$ is identically zero. For the experiments, we consider samples of size 500, 2500 and 5000.

For each of the UCGs and corresponding samples generated above, we run the parameter estimation procedure until the MLEs do not change in two consecutive iterations or 100 iterations are performed. We then compute the relative difference between each true parameter value $\theta$ and the MLE $\hat{\theta}$, which we define as $abs((\theta - \hat{\theta})/\theta)$. We also compute the residual difference, which we define as the residual with the parameter estimates minus the residual with the true parameter values. The procedure is implemented in \texttt{R} and the code is available at \texttt{https://www.dropbox.com/s/b9vmqgf99da3qxm/UCGs.R?dl=0}.

\begin{table}[t]
\begin{center}\caption{Results of the experiments with edge probability 0.2.}\label{tab:res02}
\scalebox{0.8}{
\begin{tabular}{|l|l|r|r|r|r|r|r|}
\hline
Sample size & Performance criterion & Min & Q1 & Median & Mean & Q3 & Max\\
\hline
& Number of edges $\ra$ & 5.00 & 9.00 & 11.00 & 11.12 & 13.00 & 19.00\\  
& Number of edges $\dra$ & 5.00 & 9.00 & 11.00 & 11.19 & 13.00 & 20.00\\
& Number of edges $-$ & 9.00 & 10.00 & 12.00 & 11.80 & 13.00 & 19.00\\ 
\hline 
500 & Number of iterations & 4.00 & 8.00 & 10.00 & 17.04 & 17.00 & 101.00\\ 
& $\Omega_{K,K}$ relative diff. & 0.00 & 0.05 & 0.13 & 0.93 & 0.41 & 912.64\\
& $\Omega_{K,Fa(K)}$ relative diff. & 0.00 & 0.11 & 0.27 & 2.34 & 0.65 & 8219.97\\ 
& $(\beta_K)_{K,Fa(K)}$ relative diff. & 0.00 & 0.20 & 0.50 & 16.40 & 1.20 & 506006.30\\ 
& $(\beta_K)_{K,Mo(K)}$ relative diff. & 0.00 & 0.01 & 0.02 & 0.12 & 0.05 & 131.21\\
& Residual diff. & -1743.25 & -5.07 & -2.70 & -10.80 & -1.57 & 43.04\\ 
\hline
2500 & Number of iterations & 5.00 & 8.00 & 10.00 & 16.75 & 16.00 & 101.00\\ 
& $\Omega_{K,K}$ relative diff. & 0.00 & 0.02 & 0.06 & 0.64 & 0.18 & 4202.19\\
& $\Omega_{K,Fa(K)}$ relative diff. & 0.00 & 0.05 & 0.12 & 1.02 & 0.28 & 2662.50\\ 
& $(\beta_K)_{K,Fa(K)}$ relative diff. & 0.00 & 0.08 & 0.21 & 2.04 & 0.56 & 21299.60\\ 
& $(\beta_K)_{K,Mo(K)}$ relative diff. & 0.00 & 0.00 & 0.01 & 0.06 & 0.02 & 107.23\\
& Residual diff. & -3137.71 & -4.95 & -2.61 & -13.56 & -1.64 & 104.26\\ 
\hline
5000 & Number of iterations & 4.00 & 8.00 & 10.00 & 16.68 & 16.00 & 101.00\\ 
& $\Omega_{K,K}$ relative diff. & 0.00 & 0.02 & 0.04 & 0.32 & 0.13 & 1088.67\\ 
& $\Omega_{K,Fa(K)}$ relative diff. & 0.00 & 0.03 & 0.08 & 0.70 & 0.20 & 1674.42\\ 
& $(\beta_K)_{K,Fa(K)}$ relative diff. & 0.00 & 0.05 & 0.15 & 2.10 & 0.40 & 26507.95\\ 
& $(\beta_K)_{K,Mo(K)}$ relative diff. & 0.00 & 0.00 & 0.01 & 0.04 & 0.02 & 19.88\\ 
& Residual diff. & -1131.10 & -5.11 & -2.68 & -10.13 & -1.59 & 131.23\\ 
\hline
\end{tabular}}
\end{center}
\end{table}

The results of the experiments are reported in Table \ref{tab:res02}. The difference between the quartiles Q1 and Q3 is not too big, which suggests that the column Median is a reliable summary of most of the runs. We therefore focus on this column for the rest of the analysis. Note however that some runs are exceptionally good and some others exceptionally bad, as indicated by the columns Min and Max. To avoid a bad run, one may consider a more sophisticated initialization of the parameter estimation procedure, e.g. multiple restarts.

The sample size has a clear impact on the accuracy of the MLEs, as indicated by the decrease in relative difference. For instance, half of the MLEs are less than 27 \%, 12 \% and 8 \% away from the true values for the samples sizes 500, 2500 and 5000, respectively. The effect of the sample size on the accuracy of the MLEs can also be appreciated from the fact that the residual difference does not grow with the sample size. The parameters $(\beta_K)_{K,Mo(K)}$ seem easier to estimate than $(\beta_K)_{K,Fa(K)}$, as indicated by the smaller relative difference of the former. This is not surprising since the latter may accumulate the errors in the estimation of $\Omega_{K,K}$ and $\Omega_{K,Fa(K)}$ (recall Equation \ref{eq:beta}).

Next, we repeat the experiments with an edge probability of 0.5 instead of 0.2, in order to consider denser UCGs. The results of these experiments are reported in Table \ref{tab:res05}. They lead to essentially the same conclusions as before. When comparing the two tables, we can see that the MLE accuracy is slightly worse for the denser UCGs. This is expected because the denser UCGs have more parameters to estimate from the same amount of data. However, the residual difference is better for the denser UCGs. This is again expected because the denser UCGs impose fewer constraints. All in all, both tables show that the quartile Q3 of the residual difference is negative, which indicates that the MLEs induce a better fit of the data than the true parameter values. We therefore conclude that the parameter estimation procedure proposed works satisfactorily.

\begin{table}[t]
\begin{center}\caption{Results of the experiments with edge probability 0.5.}\label{tab:res05}
\scalebox{0.8}{
\begin{tabular}{|l|l|r|r|r|r|r|r|}
\hline
Sample size & Performance criterion & Min & Q1 & Median & Mean & Q3 & Max\\
\hline
& Number of edges $\ra$ & 14.00 & 22.00 & 25.00 & 24.89 & 27.00 & 36.00\\  
& Number of edges $\dra$ & 13.00 & 23.00 & 25.00 & 25.15 & 28.00 & 36.00\\
& Number of edges $-$ & 13.00 & 20.00 & 22.00 & 22.24 & 24.00 & 32.00\\ 
\hline 
500 & Number of iterations & 7.00 & 11.00 & 16.00 & 26.21 & 28.00 & 101.00\\ 
& $\Omega_{K,K}$ relative diff. & 0.00 & 0.07 & 0.19 & 1.68 & 0.53 & 3293.27\\
& $\Omega_{K,Fa(K)}$ relative diff. & 0.00 & 0.12 & 0.31 & 2.87 & 0.73 & 18323.37\\ 
& $(\beta_K)_{K,Fa(K)}$ relative diff. & 0.00 & 0.10 & 0.28 & 1.80 & 0.73 & 2701.20\\ 
& $(\beta_K)_{K,Mo(K)}$ relative diff. & 0.00 & 0.01 & 0.02 & 0.18 & 0.06 & 1298.98\\
& Residual diff. & -2570.64 & -11.48 & -6.30 & -16.99 & -4.15 & 577.87\\ 
\hline
2500 & Number of iterations & 6.00 & 11.00 & 15.00 & 25.75 & 28.00 & 101.00\\ 
& $\Omega_{K,K}$ relative diff. & 0.00 & 0.03 & 0.09 & 0.74 & 0.23 & 1880.50\\
& $\Omega_{K,Fa(K)}$ relative diff. & 0.00 & 0.05 & 0.13 & 0.98 & 0.32 & 3092.89\\ 
& $(\beta_K)_{K,Fa(K)}$ relative diff. & 0.00 & 0.05 & 0.12 & 0.81 & 0.33 & 1138.60\\ 
& $(\beta_K)_{K,Mo(K)}$ relative diff. & 0.00 & 0.00 & 0.01 & 0.06 & 0.02 & 65.53\\
& Residual diff. & -1409.83 & -10.80 & -6.09 & -7.53 & -4.16 & 3157.38\\ 
\hline
5000 & Number of iterations & 6.00 & 11.00 & 15.00 & 25.66 & 27.25 & 101.00\\ 
& $\Omega_{K,K}$ relative diff. & 0.00 & 0.02 & 0.06 & 0.50 & 0.17 & 1236.54\\ 
& $\Omega_{K,Fa(K)}$ relative diff. & 0.00 & 0.04 & 0.10 & 0.94 & 0.23 & 4179.08\\ 
& $(\beta_K)_{K,Fa(K)}$ relative diff. & 0.00 & 0.03 & 0.09 & 0.57 & 0.24 & 969.71 \\ 
& $(\beta_K)_{K,Mo(K)}$ relative diff. & 0.00 & 0.00 & 0.01 & 0.07 & 0.02 & 644.40\\ 
& Residual diff. & -2423.18 & -11.15 & -6.35 & -2.24 & -4.22 & 5801.54\\ 
\hline
\end{tabular}}
\end{center}
\end{table}

Finally, we conduct a sanity check aimed to evaluate the behavior of the parameter estimation procedure proposed when the UCG contains spurious or superfluous edges, i.e. edges whose associated parameters are zero and thus may be removed from the UCG. To this end, we repeat the previous experiments with a slight modification. Like before, the elements of $(\beta_K)_{K,Mo(K)}$, $\Omega_{K,Fa(K)}$ and $\Omega_{K,K}$ associated to the edges in the UCG are sampled uniformly from the interval [-3,3]. However, we now set 25 \% of these parameters to zero. We expect the estimates for these zeroed parameters to get closer to zero as the sample size grows. The results of these experiments with an edge probability of 0.5 are reported in Table \ref{tab:res05zeroed}. The results for an edge probability of 0.2 are similar. The first three rows of the table report the number of edges. Each edge has associated one parameter, and 25 \% of these parameters have been zeroed in the experiments. In the next rows, the table reports the absolute difference between the zeroed parameters and their estimates, i.e. the absolute values of the estimates. As expected, the larger the sample size is the closer to zero the MLEs of the zeroed parameters become. For instance, 75 \% of MLEs of the zeroed parameters take a value smaller than 0.76, 0.33 and 0.24 for the samples sizes 500, 2500 and 5000, respectively. Note that these numbers are much lower for the zeroed $(\beta_K)_{K,Mo(K)}$ parameters (specifically 0.06, 0.03, 0.02) because, recall, our parameter estimation procedure initializes $(\hat{\beta}_K)_{K,Mo(K)}$ to zero. Table \ref{tab:res05zeroed} also shows the residual difference, which is comparable to that in Table \ref{tab:res05}, which suggests that the existence of spurious edges does not hinder fitting the data. We therefore conclude that the parameter estimation procedure proposed behaves as it should in this sanity check.

\begin{table}[t]
\begin{center}\caption{Results for the zeroed parameters with edge probability 0.5.}\label{tab:res05zeroed}
\scalebox{0.8}{
\begin{tabular}{|l|l|r|r|r|r|r|r|}
\hline
Sample size & Performance criterion & Min & Q1 & Median & Mean & Q3 & Max\\
\hline
& Number of edges $\ra$ & 14.00 & 22.00 & 25.00 & 24.89 & 27.00 & 36.00\\  
& Number of edges $\dra$ & 13.00 & 23.00 & 25.00 & 25.15 & 28.00 & 36.00\\
& Number of edges $-$ & 13.00 & 20.00 & 22.00 & 22.32 & 25.00 & 33.00\\ 
\hline 
500 & Number of iterations & 7.00 & 11.00 & 16.00 & 25.84 & 28.00 & 101.00\\ 
& $\Omega_{K,K}$ absolute diff. & 0.00 & 0.15 & 0.35 & 0.50 & 0.70 & 5.71\\
& $\Omega_{K,Fa(K)}$ absolute diff. & 0.00 & 0.17 & 0.39 & 0.54 & 0.76 & 10.11\\ 
& $(\beta_K)_{K,Mo(K)}$ absolute diff. & 0.00 & 0.01 & 0.03 & 0.05 & 0.06 & 1.21\\
& Residual diff. & -3664.56 & -11.18 & -6.13 & -19.06 & -4.15 & 691.96\\ 
\hline
2500 & Number of iterations & 6.00 & 11.00 & 15.00 & 25.30 & 27.00 & 101.00\\ 
& $\Omega_{K,K}$ absolute diff. & 0.00 & 0.06 & 0.15 & 0.21 & 0.29 & 6.36\\
& $\Omega_{K,Fa(K)}$ absolute diff. & 0.00 & 0.07 & 0.17 & 0.24 & 0.33 & 7.37\\ 
& $(\beta_K)_{K,Mo(K)}$ absolute diff. & 0.00 & 0.01 & 0.01 & 0.02 & 0.03 & 0.53\\
& Residual diff. & -1859.05 & -11.04 & -6.13 & -13.03 & -4.09 & 3222.82\\ 
\hline
5000 & Number of iterations & 6.00 & 11.00 & 15.00 & 25.66 & 27.25 & 101.00\\ 
& $\Omega_{K,K}$ absolute diff. & 0.00 & 0.04 & 0.11 & 0.16 & 0.22 & 5.62\\ 
& $\Omega_{K,Fa(K)}$ absolute diff. & 0.00 & 0.06 & 0.12 & 0.17 & 0.24 & 8.21\\  
& $(\beta_K)_{K,Mo(K)}$ absolute diff. & 0.00 & 0.00 & 0.01 & 0.01 & 0.02 & 0.56\\ 
& Residual diff. & -1256.74 & -10.93 & -6.18 & -6.18 & -4.22 & 6708.29\\ 
\hline
\end{tabular}}
\end{center}
\end{table}

\section{Causal Inference}\label{sec:docalculus}

In this section, we show how to compute the effects of interventions in UCGs. Intervening on a set of variables $X \subseteq V$ modifies the natural causal mechanism of $X$, as opposed to (passively) observing $X$. For simplicity, we only consider interventions that set $X$ to constant values. We represent an intervention that sets $X=x$ as $do(X=x)$. Given a chain component $K$ of an UCG $G$, we have that
\[
K | Pa(K) \sim \mathcal{N}(\beta_K Pa(K), \Lambda_K)
\]
as discussed at length in Section \ref{sec:ucgs}, or equivalently
\[
K = \beta_K Pa(K) + \epsilon_K
\]
where $\epsilon_K \sim \mathcal{N}(0, \Lambda_K)$. We interpret the last equation as a structural equation, i.e. it defines the natural causal mechanism of the variables in $K$. Specifically, the natural causal mechanism of $V_i \in K$ is given by the following structural equation:
\[
V_i = \sum_{V_j \in Pa(K)} \beta_{i,j} V_j + \epsilon_i
\]
where $\epsilon_i$ denotes an error variable representing the unmodelled causes of $V_i$. Then, the intervention $do(V_i=a)$ amounts to replacing the last equation with the structural equation $V_i=a$, which we dub the interventional causal mechanism of $V_i$. The resulting system of equations represents the behavior of the phenomenon being modelled under the intervention $do(V_i=a)$, i.e. $p(V \setminus \{V_i\} | do(V_i=a))$. Moreover, $p(V \setminus \{V_i\} | do(V_i=a))$ satisfies the global Markov property with respect to $G_{V \setminus \{V_i\}}$. To see it, note that every $\{V_i\}$-open route in $G$ that does not include $V_i$ is a $\{V_i\}$-open route in $G_{V \setminus \{V_i\}}$ too. However, every $\{V_i\}$-open route in $G$ that includes $V_i$ must contain a subroute of the form $A \dra V_i \dla B$ or $A \dra V_i - B$ or $A \dra V_i \la B$ or $A \ra V_1 - \cdots - V_i - \cdots - V_n \la B$. Note that all of these subroutes are part of the natural causal mechanism of $V_i$ which has been replaced and, thus, they are inactive, i.e. they are not really $\{V_i\}$-open.\footnote{Note that we cannot simply remove from $G$ the edges in these subroutes, because some are part of the natural causal mechanisms of other variables in $K$.}

The single variable interventions described in the paragraph above can be generalized to sets by simply replacing the corresponding equation for each variable in the set. 

Graphically, we can represent the natural and interventional causal mechanisms of the variables in an UCG $G$ by adding a new parent to each variable. We denote the resulting UCG by $G'$. The new parent of a variable $V_i$ is a variable $F_{V_i}$ that has the same domain as $V_i$ plus a state labeled {\it idle}: $F_{V_i} = a$ represents the intervention $do(V_i=a)$, whereas $F_{V_i} = {\it idle}$ represents no intervention on $V_i$. In other words, $F_{V_i} = a$ represents that the natural causal mechanism is inactive and the interventional causal mechanism is active, whereas $F_{V_i} = {\it idle}$ represents the opposite. See \citet[Section 3.2.2]{Pearl2009} for further details. In order to decide whether to augment $G$ with $F_{V_i} \ra V_i$ or $F_{V_i} \dra V_i$, we have to study the interventional causal mechanism. This is unlike previous works where the value set by the intervention is all that matters. Specifically, if $F_{V_i} = {\it idle}$ then the interventional causal mechanism is inactive and, thus, it does not matter whether we add $F_{V_i} \ra V_i$ or $F_{V_i} \dra V_i$. We may even decide not to add $F_{V_i}$ at all. On the other hand, if $F_{V_i} = a$ then the interventional causal mechanism is active and, thus, the type of arrow added matters: If the interventional causal mechanism is such that the intervention delivered may affect (i.e., interfere with) the rest of the variables in $K$, then we augment $G$ with $F_{V_i} \ra V_i$, otherwise we augment it with $F_{V_i} \dra V_i$. We illustrate this with the mother-child example from Section \ref{sec:motivation}. An intervention that immunizes the mother against the disease (i.e., $do(D_1=0)$) may or may not protect the child (i.e., interfere with $D_2$), depending on how the intervention is delivered: It protects the child when the interventional causal mechanism is the intake of some medication (different from the vaccination $V_1$ but comparable), but it does not protect the child when the interventional causal mechanism is a gene therapy (different from the healthy carrier genotype $G_1$ but comparable). Then, the former case should be modeled as $F_{D_1} \ra D_1$, whereas the latter should be modelled as $F_{D_1} \dra D_1$.

As mentioned, our goal is to compute or identify a causal effect $p(Y|do(X=x))$ with the help of a given UCG. That is, we would like to leverage the UCG to transform the causal effect of interest into an expression that contains neither the $do$ operator nor latent variables, so that it can be estimated from observational data. Pearl's {\em do}-calculus consists of three rules whose repeated application together with standard probability manipulations do the job for acyclic directed mixed graphs \citep[Section 3.4]{Pearl2009}. We show next that the calculus carries over into UCGs. Specifically, the calculus consists of the following three rules:
\begin{itemize}
\item Rule 1 (insertion/deletion of observations):
\[
p(Y| do(X), Z \cup W) = p(Y| do(X), W) \text{ if } Y \ci_{(G_{V \setminus X})'} Z | W.
\]

\item Rule 2 (intervention/observation exchange): 
\[
p(Y| do(X \cup Z), W) = p(Y| do(X), Z \cup W) \text{ if } Y \ci_{(G_{V \setminus X})'} F_Z | W \cup Z.
\]

\item Rule 3 (insertion/deletion of interventions): 
\[
p(Y| do(X \cup Z), W) = p(Y| do(X), W) \text{ if } Y \ci_{(G_{V \setminus X})'} F_Z | W.
\]
\end{itemize}

\begin{theorem}\label{the:soundrules}
Rules 1-3 are sound for UCGs.
\end{theorem}

\begin{corollary}\label{cor:lwf}
Consider an intervention on a set of variables $X$ in an UCG $G$. If the interventional causal mechanism of each variable $V_i \in X$ implies interference (i.e., it is modelled by augmenting $G$ with the edge $F_{V_i} \ra V_i$), then
\[
p(V \setminus X | do(X))= \prod_{K \in Cc(G)} p(K \setminus X |Pa(K) \cup (K \cap X)).
\]
\end{corollary}

Note that the corollary above implies that the causal effect of any intervention that involves interference is identifiable in a parametrized UCG. The parameter values may be provided by an expert or estimated from data as shown in Section \ref{sec:MLEs}. In the latter case, all the variables in the model are assumed to be measured in the data. When the model contains latent variables, we may still perform causal effect identification via rules 1-3. It is also worth mentioning that the corollary above is generalization of causal effect identification in LWF CGs as proposed by \citet{LauritzenandRichardson2002}, which in turn is a generalization of causal effect identification in DAGs \citep{Pearl2009}. This may come as a surprise because undirected edges in UCGs represent interference relationships whereas, in the work by \citet{LauritzenandRichardson2002}, they represent dependencies in the equilibrium distribution of a dynamic system with feedback loops. However, it has been suggested that interference is nothing but dependencies in an equilibrium distribution \citep{OgburnandVanderWeele2014,Ogburnetal.2018,Shpitser2015}.

\section{Identifiability of LWF and AMP CGs}\label{sec:identifiability}

This section proves that identifiability of LWF and AMP CGs is possible when the error variables have equal variance. The error variable $\epsilon_A$ associated with a variable $A \in V$ represents the unmodelled causes of $A$. In other words, given a probability distribution $p$ that is faithful to a LWF or AMP CG $G$, we can identify the Markov equivalence class of $G$ (recall Theorem \ref{the:eq}) from $p$ by running, for instance, the learning algorithm developed by \citet{Studeny1997b} for LWF CGs and by \citet{Penna2012} for AMP CGs. We prove below that we can actually identify $G$ from $p$ if the error variables have equal variance. We discuss this assumption at the end of the section. Our result generalizes a similar result reported by \citet{PetersandBuhlmann2014} for DAGs.

Specifically, let $G$ denote a LWF or AMP CG. Assume that the non-zero entries of $\beta_K$ and $\Lambda_K^{-1}$ in Equation \ref{eq:recursion2} have been selected at random. This implies that $p$ is faithful to $G$ with probability almost 1 by \citet[Theorems 1 and 2]{Penna2011} and \citet[Theorem 6.1]{Levitzetal.2001}.\footnote{\citet{Penna2011} proves this result for LWF CGs using a different parameterization than $\beta_K$ and $\Lambda_K^{-1}$. However, there is a one-to-one mapping between both parameterizations by Equations \ref{eq:beta} and \ref{eq:lambda}. So, his result applies to the parameterization used in this paper.} We therefore assume faithfulness hereinafter. Moreover, we rewrite Equation \ref{eq:recursion2} as
\begin{equation}\label{eq:recursion2b}
K = \beta_K Pa(K) + \epsilon_K
\end{equation}
in distribution, where $\epsilon_K \sim \mathcal{N}(0, \Lambda_K)$. For any $V_i \in V$, we have then that
\begin{equation}\label{eq:recursion2c}
V_i = \sum_{V_j \in Pa(K)} \beta_{i,j} V_j + \epsilon_i
\end{equation}
in distribution, where $K$ denotes the chain component of $G$ that contains $V_i$, and $\epsilon_i$ denotes an error variable representing the unmodelled causes of $V_i$. All such error variables are jointly denoted by $\epsilon$ which is distributed according to $\mathcal{N}(0, \Lambda)$, where $\Lambda$ is a block diagonal matrix whose blocks correspond to the covariance matrices $\Lambda_K$ for all $K \in Cc(G)$. Moreover, we assume that the errors $\epsilon_i$ have equal but unknown variance $\lambda^2$. Note that if the error variances are unequal and unknown but have the form $\Lambda_{i,i} = \lambda_i^2 \lambda^2$ for some known ratios $\lambda_i^2$, then we can satisfy the equal error variance assumption by rescaling each variable $V_i$ by dividing it with $\lambda_i$, i.e. $V_i \mapsto V_i / \lambda_i$. This implies that the linear coefficients and the errors get rescaled as $\beta_{i,j} \mapsto \beta_{i,j} \lambda_j / \lambda_i$ and $\epsilon_i \mapsto \epsilon_i / \lambda_i$. The following lemma proves that, after the rescaling, the error covariance matrix is still positive definite and keeps all the previous (in)dependencies which implies that, after the rescaling, $p$ is still faithful to $G$ with probability almost 1.

\begin{lemma}\label{lem:rescaling}
Consider the rescaling $\epsilon_i \mapsto \epsilon_i / \lambda_i$ for all $i$. Then, the error covariance matrix represents the same independences before and after the rescaling. Moreover, the error covariance matrix is positive definite after the rescaling if and only if it was so before the rescaling.
\end{lemma}

We are now ready to state formally the main result of this section.

\begin{theorem}\label{the:main}
Let $p$ be a Gaussian distribution generated by Equation \ref{eq:recursion2c} with equal error variances. Then, $G$ is identifiable from $p$.
\end{theorem}

Note that the theorem above implies that two LWF CGs or AMP CGs that represent the same separations are not Markov equivalent under the constraint of equal error variances, i.e. Theorem \ref{the:eq} does not hold under this constraint. The suitability of the equal error variances assumption should be assessed on a per domain basis. However, it may not be unrealistic to assume that it holds when the variables correspond to a relatively homogeneous set of individuals. For instance, in the case of a contagious disease, the error variable represents the unmodelled causes of an individual developing the disease, e.g. environmental factors. We may assume that these factors are the same for all the individuals, given their homogeneity. Therefore, we may assume equal error variances. We conjecture that the theorem above also holds for UCGs. However, a formal proof of this result requires first a characterization of Markov equivalence for UCGs, something that we postpone to a future article.

\section{Discussion}\label{sec:discussion}

LWF and AMP CGs are typically used to represent independence models. However, they can also be used to represent causal models. For instance, LWF CGs have been shown to be suitable for representing the equilibrium distributions of dynamic systems with feedback loops \citep{LauritzenandRichardson2002}. AMP CGs have been shown to be suitable for representing causal linear models with additive Gaussian noise \citep{Penna2016}. LWF CGs have been extended into segregated graphs, which have been shown to be suitable for representing causal models with interference \citep{Shpitser2015}. In this paper, we have shown how to combine LWF and AMP CGs to represent causal models of domains with both interference and non-interference relationships. Moreover, we have defined global, local and pairwise Markov properties for the new models, which we have coined unified chain graphs (UCGs), and shown that these properties are equivalent for Gaussian distributions. We have also proposed and evaluated an algorithm for computing MLEs of the parameters of an UCG. Finally, we have shown how to perform causal inference in UCGs. 

It is worth mentioning that we are not the first to unify LWF and AMP CGs. \cite{LauritzenandSadeghi2018} recently proposed a new class of graphical models that unify many existing classes, including LWF and AMP CGs. Specifically, they consider acyclic graphs with four types of edges: Directed edges, bidirected edges, and solid and dotted undirected edges. Several edges between any pair of nodes are allowed. If their graphs only contain directed and solid undirected edges, then they coincide with LWF CGs. If they only contain directed and dotted undirected edges, then they coincide with AMP CGs. The authors develop global and pairwise Markov properties for the new models, and prove their equivalence. However, the pairwise Markov property only applies to graphs that have no dotted undirected edges. So, it does not apply to AMP CGs or superclasses of it. Other differences with our work are that no local Markov property is proposed, no parameterization or parameter learning algorithm is proposed, and no causal interpretation is given. However, the main difference with our work is that their models cannot accommodate both interference and non-interference relationships, because they rely on a single type of directed edge. For instance, our mother-child example may be modeled with a graph that contains the edges $V_1 \ra D_1$, $G_1 \ra D_1$ and $D_1 - D_2$ or $D_1 \cdots D_2$ or both. However, this graph cannot represent that intervening on $V_1$ must have an effect on $D_2$ while intervening on $G_1$ must not, because the paths from $V_1$ and $G_1$ to $D_2$ contain the same types of edges in the same order, namely a directed edge followed by a solid or dotted undirected edge. This leads us to conclude that the models proposed by \cite{LauritzenandSadeghi2018} do not subsume UCGs.

Finally, we would like to mention some questions that we have not studied in this paper but which we will. We plan to extend UCGs to categorical random variables. When dealing with continuous random variables, assuming that these are jointly Gaussian simplifies the problem by restricting the relations to be linear. However, in our opinion, the main simplification that the Gaussian assumption brings is that checking whether an independence holds reduces to checking whether a linear coefficient or an entry in a precision matrix is identically zero. Discrete UCGs will not enjoy this advantage. In any case, finding a suitable/amenable parameterization of discrete UCGs is of utmost importance. Moreover, \citet{Drton2009} has shown that discrete LWF CGs are smooth models but discrete AMP CGs are not. Non-smoothness implies that some standard asymptotic distribution results (e.g., normal distribution limits for MLEs and $\chi^2$-limits for likelihood ratios) may not hold for the model at hand. Therefore, we need to investigate whether non-smoothness hinders discrete UCGs from representing interference and non-interference relationships. Other questions that we plan to study are (i) characterize Markov equivalent UCGs, (ii) develop structure learning algorithms based on the local and pairwise Markov properties, (iii) extend UCGs to model confounding via bidirected edges, and (iv) make use of the linearity of the relations for causal effect identification in UCGs along the lines in \citet[Chapter 5]{Pearl2009}.

\section*{Acknowledgments}

We thank the Reviewers and the Associate Editor for their comments, which helped us to improve this work substantially.

\section*{Appendix A: Proofs}\label{sec:appendix}

\begin{proof}[Proof of Theorem \ref{the:eq}]
The result has been proven before for LWF CGs \citep[Corollary 1]{Penna2011}. We prove it here for AMP CGs. The if part is trivial. To see the only if part, note that there are Gaussian distributions $p$ and $q$ that are faithful respectively to $G$ and $H$ \citep[Theorem 6.1]{Levitzetal.2001}. Moreover, $p$ satisfies the global Markov property with respect to $H$, because $G$ and $H$ are Markov equivalent. Likewise for $q$ and $G$. Therefore, $G$ and $H$ must represent the same separations.
\end{proof}

\begin{proof}[Proof of Lemma \ref{lem:requirement}]
As discussed before, missing undirected edges impose zero restrictions on the elements of $\Lambda_K^{-1}$ due to Equation \ref{eq:zeron}. Likewise, missing $\dra$ edges impose zero restrictions on the elements of $\beta_K$ corresponding to the mothers due to Equation \ref{eq:zerom}. Likewise, missing $\ra$ edges impose zero restrictions on the elements of $\Omega_{K,Pa(K)}$ corresponding to the fathers due to Equation \ref{eq:zerof}, but not on the elements of $\beta_K$. Therefore, if we ignore the requirement and do not make any distinction between mothers and fathers, then the missing $\dra$ edges from $Pa(K)$ to $K$ impose zero restrictions on $\beta_K$ and, thus, we can simply ignore the edges $\ra$ since this will not imply any additional zero restriction.
\end{proof}

Given an UCG $G$, let $G^U$ denote the subgraph of $G$ resulting from the following process: Start with the empty graph over $U$ and add to it the edge $A \to B$ if and only if $G$ has a route of the form $A \to B \to \cdots \to C$ where (i) $C \in U$, and (ii) the route has no subroute $R \dra S - T$. The next lemma shows that if $X \ci_G Y | Z$, then $X$ and $Y$ can be extended to cover the whole $V(G^{X \cup Y \cup Z}) \setminus Z$, where $V(G^{X \cup Y \cup Z})$ denotes the nodes in $G^{X \cup Y \cup Z}$.

\begin{lemma}\label{lem:At3}
Given an UCG $G$ and three disjoint subsets $X$, $Y$ and $Z$ of $V$ such that $X \ci_G Y | Z$, then $X' \ci_{G^{X \cup Y \cup Z}} Y' | Z$ with $Y'= \{B' \in V(G^{X \cup Y \cup Z}) \setminus (X \cup Z) \: : \: X \ci_{G^{X \cup Y \cup Z}} B' | Z \}$ and $X' = V(G^{X \cup Y \cup Z}) \setminus (Y' \cup Z)$.
\end{lemma}

\begin{proof}
We show that $A' \ci_{G^{X \cup Y \cup Z}} B' | Z$ for all $A' \in X'$ and $B' \in Y'$. Note that this holds if $A' \in X$ by definition of $Y'$. Now, consider $A' \in X' \setminus X$. Assume to the contrary that there exists a $Z$-open route $\pi$ between $A'$ and $B' \in Y'$ in $G^{X \cup Y \cup Z}$. Moreover, there is a $Z$-open route $\sigma$ between $A'$ and some $A \in X$ in $G^{X \cup Y \cup Z}$, because $A' \notin Y'$. Let $\rho=\sigma \cup \pi$, i.e. the route resulting from concatenating $\sigma$ and $\pi$. Note that if $A'$ is neither a collider node in $\rho$ nor in a collider section of $\rho$, then $\rho$ is $Z$-open because $A' \notin Z$. This contradicts that $B' \in Y'$ and, thus, $A'$ must be (i) in a collider section of $\rho$, or (ii) a collider node in $\rho$. Note that in case (i), the collider section is of the form $V_1 \ra V_2 - \cdots - A' - \cdots - V_{n-1} \la V_n$ where $V_2, \ldots, V_{n-1} \notin Z$ because, otherwise, $\pi$ or $\sigma$ are not $Z$-open. This together with the fact that $A' \notin Z$ imply that $\rho$ is not $Z$-open in either case (i) or (ii). However, it can be modified into a $Z$-open route between $A$ and $B'$ as follows, which contradicts that $B' \in Y'$.
\begin{itemize}
\item Assume that case (i) holds. Recall that $A' \in V(G^{X \cup Y \cup Z})$ and consider the following three cases. First, assume that $A' \in V(G^Z)$. Then, $G^{X \cup Y \cup Z}$ has a route $\varrho$ from $A'$ to some $C \in Z$ that only contains edges $\ra$, $\dra$ and $-$ and no subroute $R \dra S - T$. Assume without loss of generality that $C$ is the only node in $\varrho$ that belongs to $Z$. Let $\varrho'$ denote the route resulting from traversing $\varrho$ from $C$ to $A'$. Then, the route $\sigma \cup \varrho \cup \varrho' \cup \pi$ is $Z$-open, which contradicts that $B' \in Y'$. Second, assume that $A' \in V(G^X)$ but $A' \notin V(G^Z)$. Then, $G^{X \cup Y \cup Z}$ has a route $\varrho$ from some $A'' \in X$ to $A'$ that only contains edges $\la$, $\dla$ and $-$ and no subroute $R - S \dla T$. Note that no node in $\varrho$ is in $Z$, because $A' \notin V(G^Z)$. Then, the route $\varrho \cup \pi$ is $Z$-open, which contradicts that $B' \in Y'$. Third, assume that $A' \in V(G^Y)$ but $A' \notin V(G^Z)$. Then, we can reach a contradiction much in the same way as in the previous case.

\item Assume that case (ii) holds. We can prove this case much in the same way as case (i). Simply note that we can now choose the route $\varrho$ in case (i) so that it does not start with an edge $-$. To see this, note that that $A'$ is a collider node in $\rho$ implies that an edge $A'' \dra A'$ is in $\rho$ and, thus, in $G^{X \cup Y \cup Z}$. However, since $A' \notin X \cup Y \cup Z$, this is possible only if $G^{X \cup Y \cup Z}$ has a route $A'' \dra A' \dra C \to \cdots \to D$ or $A'' \dra A' \ra C \to \cdots \to D$ where (i) $D \in X \cup Y \cup Z$, and (ii) the route has no subroute $R \dra S - T$.
\end{itemize}
\end{proof}

A pure collider route is a route whose all intermediate nodes are collider nodes or in a collider section of the route, i.e. $A \dra C \dla B$, $A \dra C - B$, $A \dra C \la B$, $A \dra C_1 - C_2 \dla B$, $A \ra C_1 - \cdots - C_n \la B$, or $A \to B$.

\begin{lemma}\label{lem:pure}
Given an UCG $G$ and three disjoint subsets $X$, $Y$ and $Z$ of $V$ such that $X \ci_G Y | Z$, then there is no pure collider route in $G^{X \cup Y \cup Z}$ between some $A' \in X'$ and $B' \in Y'$.
\end{lemma}

\begin{proof}
Note that $X' \ci_{G^{X \cup Y \cup Z}} Y' | Z$ with $V(G^{X \cup Y \cup Z}) = X' \cup Y' \cup Z$ by Lemma \ref{lem:At3}. Assume to the contrary that there is a pure collider route $\rho$ in $G^{X \cup Y \cup Z}$ between $A' \in X'$ and $B' \in Y'$. If $\rho$ is of the form $A' \to B'$, then it contradicts that $X' \ci_{G^{X \cup Y \cup Z}} Y' | Z$. Likewise, if $\rho$ is of the form $A' \dra C \dla B'$, $A' \dra C - B'$ or $A' \dra C \la B'$, then it contradicts that $X' \ci_{G^{X \cup Y \cup Z}} Y' | Z$ regardless of whether $C \in X'$, $C \in Y'$ or $C \in Z$. Likewise, if $\rho$ is of the form $A' \dra C_1 - C_2 \dla B'$, then $C_1 \in X'$ or $C_1 \in Z$ to avoid contradicting that $X' \ci_{G^{X \cup Y \cup Z}} Y' | Z$. For the same reason, $C_2 \in Y'$ or $C_2 \in Z$. However, any of the four combinations contradicts that $X' \ci_{G^{X \cup Y \cup Z}} Y' | Z$. Finally, if $\rho$ is of the form $A' \ra C_1 - \cdots - C_n \la B'$, then $C_1, \ldots, C_n \notin Z$ to avoid contradicting that $X' \ci_{G^{X \cup Y \cup Z}} Y' | Z$. However, this implies that some node in $X'$ is adjacent to some node in $Y'$, which contradicts that $X' \ci_{G^{X \cup Y \cup Z}} Y' | Z$.
\end{proof}

\begin{lemma}\label{lem:marginal}
Let $U \sim \mathcal{N}(0, \Lambda_U)$ and $L \sim \mathcal{N}(\beta_L U, \Lambda_L)$ where $\beta_L$ is of dimension $|L| \times |U|$ and, thus, $W=(U, L)^T \sim \mathcal{N}(0, \Sigma)$. If we now set $(\beta_L)_{i,j}=0$, then $W \sim \mathcal{N}(0, \Sigma')$ such that $\Sigma'_{m,n} = \Sigma_{m,n}$ for all $m,n \neq i$.
\end{lemma}

\begin{proof}
\citet[Section 2.3.3]{Bishop2006} proves that
\[
\Sigma = 
\begin{pmatrix} 
\Lambda_U & \Lambda_U \beta_L^T\\
\beta_L \Lambda_U & \Lambda_L + \beta_L \Lambda_U \beta_L^T  
\end{pmatrix}.
\]
Now, note that
\[
(\beta_L \Lambda_U)_{m,n}= \sum_r (\beta_L)_{m,r} (\Lambda_U)_{r,n}
\]
and
\[
(\Lambda_L + \beta_L \Lambda_U \beta_L^T )_{m,n}= (\Lambda_L)_{m,n} + \sum_r \sum_s (\beta_L)_{m,r} (\Lambda_U)_{r,s} (\beta_L)_{n,s}
\]
which imply the result.
\end{proof}

The following observation follows from the lemma above and will be used later. Let $p(W)$ and $p'(W)$ denote the distribution of $W$ before and after setting $(\beta_L)_{i,j}=0$ in the lemma, i.e. $\mathcal{N}(0, \Sigma)$ and $\mathcal{N}(0, \Sigma')$. Then, the lemma implies that $p(W \setminus \{i\})=p'(W \setminus \{i\})$.

\begin{lemma}\label{lem:zero}
Let $K_1, \ldots, K_n$ denote the topologically sorted chain components of an UCG $G$. Then, $(K_1, \ldots, K_n)^T \sim \mathcal{N}(0, \Sigma)$ where $(\Sigma^{-1})_{i,j}=0$ if there are no pure collider routes in $G$ between $i$ and $j$.
\end{lemma}

\begin{proof}
We prove the result by induction on the number of chain components $n$. The result clearly holds for $n=1$, because the only possible pure collider route in $G$ between $i$ and $j$ is of the form $i - j$, and Equation \ref{eq:zeron} implies that $(\Sigma^{-1})_{i,j}=0$ if $i - j$ is not in $G$. Assume as induction hypothesis that the result holds for all $n<r$. We now prove it for $n=r$. Let $U= K_1 \cup \cdots \cup K_{r-1}$ and $L=K_r$. By the induction hypothesis, $U \sim \mathcal{N}(0, \Lambda_U)$. Moreover, let $L \sim \mathcal{N}(\beta_L Pa(L), \Lambda_L)$. \citet[Section 2.3.3]{Bishop2006} shows that $(U, L)^T \sim \mathcal{N}(0, \Sigma)$ where
\begin{equation}\label{eq:zero}
\Sigma^{-1} = 
\begin{pmatrix} 
\Lambda_{U}^{-1}+\beta_{L}^T \Lambda_{L}^{-1} \beta_{L} & -\beta_{L}^T \Lambda_{L}^{-1}\\
-\Lambda_{L}^{-1} \beta_{L} & \Lambda_{L}^{-1}  
\end{pmatrix}.
\end{equation}

Assume that $i,j \in U$. Then, Equation \ref{eq:zero} implies that $(\Sigma^{-1})_{i,j}=0$ if $(\Lambda_U^{-1})_{i,j}=0$ and $(\beta_{L}^T \Lambda_{L}^{-1} \beta_{L})_{i,j}=0$. We show below that the first (respectively second) condition holds if there are no pure collider routes in $G$ between $i$ and $j$ through nodes in $U$ (respectively $L$).

\begin{itemize}
\item By the induction hypothesis, $(\Lambda_U^{-1})_{i,j}=0$ if there are no pure collider routes in $G$ between $i$ and $j$ through nodes in $U$.

\item Note that $( \beta_{L}^T \Lambda_{L}^{-1} \beta_{L} )_{i,j}=\sum_r \sum_s (\beta_L)_{r,i} (\Lambda_{L}^{-1})_{r,s} (\beta_L)_{s,j}=0$ if $(\beta_L)_{r,i}$ $(\Lambda_{L}^{-1})_{r,s} (\beta_L)_{s,j}=0$ for all $r$ and $s$. For $r=s$, $(\beta_L)_{r,i} (\Lambda_{L}^{-1})_{r,s} (\beta_L)_{s,j}=0$ if $(\beta_L)_{r,i}=0$ or $(\beta_L)_{s,j}=0$. For $r \neq s$, $(\beta_L)_{r,i} (\Lambda_{L}^{-1})_{r,s} (\beta_L)_{s,j}=0$ if $(\beta_L)_{r,i}=0$ or $(\Lambda_{L}^{-1})_{r,s}=0$ or $(\beta_L)_{s,j}=0$. Moreover, as shown in Section \ref{sec:ucgs}, $(\beta_L)_{r,i}=0$ if neither $i \dra r$ nor $i \ra c - \cdots - r$ is in $G$. Likewise, $(\beta_L)_{s,j}=0$ if neither $j \dra s$ nor $j \ra d - \cdots - s$ is in $G$. Finally, $(\Lambda_{L}^{-1})_{r,s}=0$ if $r - s$ is not in $G$. These conditions rule out the existence of pure collider routes in $G$ between $i$ and $j$ through nodes in $L$.
\end{itemize}

Now, assume that $i \in U$ and $j \in L$. Then, Equation \ref{eq:zero} implies that $(\Sigma^{-1})_{i,j}=0$ if $(\beta_{L}^T \Lambda_{L}^{-1})_{i,j}=0$ if $\sum_r (\beta_L)_{r,i} (\Lambda_{L}^{-1})_{r,j}=0$ if $(\beta_L)_{r,i} (\Lambda_{L}^{-1})_{r,j}=0$ for all $r$. For $j=r$, $(\beta_L)_{r,i} (\Lambda_{L}^{-1})_{r,j}=0$ if $(\beta_L)_{r,i}=0$ if neither $i \dra r$ nor $i \ra c - \cdots - r$ is in $G$. For $j \neq r$, $(\beta_L)_{r,i} (\Lambda_{L}^{-1})_{r,j}=0$ if $(\beta_L)_{r,i}=0$ or $(\Lambda_{L}^{-1})_{r,j}=0$ if neither $i \dra r$ nor $i \ra c - \cdots - r$ is in $G$, or $r - j$ is not in $G$. Either case holds if there are no pure collider routes in $G$ between $i$ and $j$.

Finally, assume that $i,j \in L$. Then, Equation \ref{eq:zero} implies that $(\Sigma^{-1})_{i,j}=0$ if $(\Lambda_L^{-1})_{i,j}=0$ if $i - j$ is not in $G$, i.e. if there are no pure collider routes in $G$ between $i$ and $j$.
\end{proof}

\begin{proof}[Proof of Theorem \ref{the:zg}]
We prove first the if part. For any chain component $K \in Cc(G)$, clearly $K \ci_G Nd(K) \setminus K \setminus Pa(K) | Pa(K)$ and thus $K \ci_p Nd(K) \setminus K \setminus Pa(K) | Pa(K)$, which implies Equation \ref{eq:recursion1}. For any vertex $i \in K$, clearly $i \ci_G Mo(K) \setminus Mo(i) | Fa(K) \cup Mo(i)$ and thus $i \ci_p Mo(K) \setminus Mo(i) | Fa(K) \cup Mo(i)$, which implies Equation \ref{eq:zerom}. For any vertices $i \in K$ and $j \in Fa(K) \setminus Fa(i)$, clearly $i \ci_G j | K \setminus \{i\} \cup Pa(K) \setminus \{j\}$ and thus $i \ci_G j | K \setminus \{i\} \cup Pa(K) \setminus \{j\}$, which implies Equation \ref{eq:zerof}. Finally, for any non-adjacent vertices $i,j \in K$, clearly $i \ci_G j | K \setminus \{i,j\} \cup Pa(K)$ and thus $i \ci_p j | K \setminus \{i,j\} \cup Pa(K)$, which implies Equation \ref{eq:zeron}.

We now prove the only if part. Consider three disjoint subsets $X$, $Y$ and $Z$ of $V$ such that $X \ci_G Y | Z$. Let $K_1, \ldots, K_n$ denote the topologically sorted chain components of $G^{X \cup Y \cup Z}$. Note that $G_{K_1 \cup \cdots \cup K_n}$ and $G^{X \cup Y \cup Z}$ only differ in that the latter may not have all the edges $\dra$ in the former. Note also that $K_1, \ldots, K_n$ are chain components of $G$, too. Consider a topological ordering of the chain components of $G$, and let $Q_1, \ldots, Q_m$ denote the components that precede $K_n$ in the ordering, besides $K_1, \ldots, K_{n-1}$. Note that the edges in $G$ from any $Q_i$ to any $K_j$ must be of the type $\dra$ because, otherwise, $Q_i$ would be a component of $G^{X \cup Y \cup Z}$. Therefore, $G_{Q_1 \cup \cdots \cup Q_m \cup K_1 \cup \cdots \cup K_n}$ and $G_{Q_1 \cup \cdots \cup Q_m} \cup G^{X \cup Y \cup Z}$ only differ in that the latter may not have all the edges $\dra$ in the former. In other words, the latter may impose additional zero restrictions on the elements of some $\beta_{K_j}$ corresponding to the mothers. Consider adding such additional restrictions to the marginal distribution $p(Q_1, \ldots, Q_m, K_1, \ldots, K_n)$ obtained from $p$ via Equations \ref{eq:recursion1} and \ref{eq:recursion2}, i.e. consider setting the corresponding elements of $\beta_{K_j}$ to zero (recall that $\beta_{K_j}$ are such that the mean vector of $p(K_j | Pa(K_j))$ is a linear function of $Pa(K_j)$ with coefficients $\beta_{K_j}$). Call the resulting distribution $p'(Q_1, \ldots, Q_m, K_1, \ldots, K_n)$.

Finally, recall again that $G_{Q_1 \cup \cdots \cup Q_m \cup K_1 \cup \cdots \cup K_n}$ and $G_{Q_1 \cup \cdots \cup Q_m} \cup G^{X \cup Y \cup Z}$ only differ in that the latter may not have all the edges $\dra$ in the former. Note also that every node in $X \cup Y \cup Z$ has the same mothers in $G_{Q_1 \cup \cdots \cup Q_m \cup K_1 \cup \cdots \cup K_n}$ and $G_{Q_1 \cup \cdots \cup Q_m} \cup G^{X \cup Y \cup Z}$. Then, $p(X \cup Y \cup Z) = p'(X \cup Y \cup Z)$ by Lemma \ref{lem:marginal} and, thus, $X \ci_p Y | Z$ if and only if $X \ci_{p'} Y | Z$. Now, recall from Lemma \ref{lem:pure} that $X \ci_G Y | Z$ implies that there is no pure collider route in $G^{X \cup Y \cup Z}$ between any vertices $A' \in X'$ and $B' \in Y'$. Then, there is no such a route either in $G_{Q_1 \cup \cdots \cup Q_m} \cup G^{X \cup Y \cup Z}$, because this UCG has no edges between the nodes in $V(G_{Q_1 \cup \cdots \cup Q_m})$ and the nodes in $V(G^{X \cup Y \cup Z})$. This implies that $A' \ci_{p'} B' | X' \setminus \{A'\} \cup Y' \setminus \{B'\} \cup Z \cup Q_1 \cup \cdots \cup Q_m$ by Lemma \ref{lem:zero} and \citet[Proposition 5.2]{Lauritzen1996}, which implies $X' \ci_{p'} Y' | Z \cup Q_1 \cup \cdots \cup Q_m$ by repeated application of intersection. Moreover, $X' \ci_{p'} Q_1 \cup \cdots \cup Q_m | Z$ because $G_{Q_1 \cup \cdots \cup Q_m} \cup G^{X \cup Y \cup Z}$ has no edges between the nodes in $V(G_{Q_1 \cup \cdots \cup Q_m})$ and the nodes in $V(G^{X \cup Y \cup Z})$. This implies $X \ci_{p'} Y | Z$ by contraction and decomposition which, as shown, implies $X \ci_p Y | Z$.
\end{proof}

\begin{proof}[Proof of Theorem \ref{the:JonesWestWright}]
We prove the result by induction on the number of chain components $n$. The result holds for $n=1$ by Equation \ref{eq:jones} \citep[Theorem 1]{JonesandWest2005}. Assume as induction hypothesis that the result holds for all $n<m$. We now prove it for $n=m$. Let $U = K_1 \cup \cdots \cup K_{m-1}$ and $L=K_m$. Let $U \sim \mathcal{N}(0, \Lambda_U)$ and $L \sim \mathcal{N}(\beta_L U, \Lambda_L)$ and, thus, $(U, L)^T \sim \mathcal{N}(0, \Sigma)$. \citet[Section 2.3.3]{Bishop2006} proves that
\[
\Sigma = 
\begin{pmatrix} 
\Lambda_U & \Lambda_U \beta_L^T\\
\beta_L \Lambda_U & \Lambda_L + \beta_L \Lambda_U \beta_L^T  
\end{pmatrix}.
\]

Now, note that
\[
(\beta_L \Lambda_U)_{i,j}= \sum_k (\beta_L)_{i,k} (\Lambda_U)_{k,j}.
\]
By the induction hypothesis, $(\Lambda_U)_{k,j}$ can be written as a sum of products of weights over the edges of the open paths between $k$ and $j$ in $G$. Moreover, as discussed in Section \ref{sec:ucgs} in relation to Equation \ref{eq:jones}, $(\beta_L)_{i,k}$ can be written as a sum of products of weights over the edges of the paths from $k$ to $i$ through nodes in $L$. Since the latter paths start all with a directed edge out of $k$, the previous observations together imply the desired result.

Finally,
\[
(\Lambda_L + \beta_L \Lambda_U \beta_L^T )_{i,j}= (\Lambda_L)_{i,j} + \sum_k \sum_l (\beta_L)_{i,k} (\Lambda_U)_{k,l} (\beta_L)_{j,l}.
\]
As discussed in Section \ref{sec:ucgs}, $(\Lambda_L)_{i,j}$ can be written as a sum of products of weights over the edges of the paths between $i$ and $j$ through nodes in $L$. By the induction hypothesis, $(\Lambda_U)_{k,l}$ can be written as a sum of products of weights over the edges of the open paths between $k$ and $l$ in $G$. Again, as discussed in Section \ref{sec:ucgs}, $(\beta_L)_{i,k}$ (respectively, $(\beta_L)_{j,l}$) can be written as a sum of products of weights over the edges of the paths from $k$ to $i$ (respectively, from $l$ to $j$) through nodes in $L$. Since the latter paths start all with a directed edge out of $k$ (respectively, out of $l$), the previous observations together imply the desired result.
\end{proof}

\begin{proof}[Proof of Theorem \ref{the:zb}]
We prove first the if part. The property B1 together with weak union and composition imply $K \ci_p Nd(K) \setminus K \setminus Pa(K) | Pa(K)$ and thus Equation \ref{eq:recursion1}. Moreover, B1 implies $i \ci_p Mo(K) \setminus Mo(i) | Fa(K) \cup Mo(i)$ by decomposition and, thus, Equation \ref{eq:zerom}. Moreover, B2 implies $i \ci_p Fa(K) \setminus Fa(i) | K \setminus \{i\} \cup Fa(i) \cup Mo(K)$ by decomposition, which implies $i \ci_p j | K \setminus \{i\} \cup Pa(K) \setminus \{j\}$ for all $j \in Fa(K) \setminus Fa(i)$ by weak union, which implies Equation \ref{eq:zerof}. Finally, B3 implies Equation \ref{eq:zeron}.

We now prove the only if part. Note that if $p$ satisfies Equations \ref{eq:recursion1} and \ref{eq:zerom}-\ref{eq:zeron} with respect to $G$, then it satisfies the global Markov property with respect to $G$ by Theorem \ref{the:zg}. Now, it is easy to verify that the block-recursive property holds. Specifically, the separations $i \ci_G Nd(K) \setminus K \setminus Fa(K) \setminus Mo(i) | Fa(K) \cup Mo(i)$ for all $i \in K$ with $K \in Cc(G)$ imply B1 by the global Markov property. Likewise, the separations $i \ci_G Nd(K) \setminus K \setminus Fa(i) \setminus Mo(K) | K \setminus \{i\} \cup Fa(i) \cup Mo(K)$ for all $i \in K$ with $K \in Cc(G)$ imply B2 by the global Markov property. Finally, the separations $i \ci_G j | K \setminus \{i,j\} \cup Pa(K)$ for all $i, j \in K$ with $K \in Cc(G)$ and such that $i - j$ is not in $G$ imply B3 by the global Markov property.
\end{proof}

\begin{proof}[Proof of Theorem \ref{the:pb}]
Note that $Nd(i)=Nd(K)$. Then, the properties P1 and P2 imply respectively B1 and B2 by repeated application of intersection. Similarly, P2 implies $i \ci_p Nd(K) \setminus \{i\} \setminus Pa(K) \setminus Ne(i) | Pa(K) \cup Ne(i)$ by repeated application of intersection, which implies $i \ci_p K \setminus \{i\} \setminus Ne(i) | Pa(K) \cup Ne(i)$ by decomposition, which implies B3 by weak union. Finally, the property B1 implies P1 by weak union.  Likewise, B2 implies P2 by weak union if $j \notin K$. Assume now that $j \in K$ and note that B2 implies $i \ci_p Nd(K) \setminus K \setminus Pa(K) | K \setminus \{i\} \cup Pa(K)$ by weak union. Note also that B3 implies that $p(K|Pa(K))$ satisfies the global Markov property with respect to $G_K$ by \citet[Theorem 3.7]{Lauritzen1996} and, thus, $i \ci_p K \setminus \{i\} \setminus Ne(i) | Pa(K) \cup Ne(i)$. Then, $i \ci_p Nd(K) \setminus \{i\} \setminus Pa(K) \setminus Ne(i) | Pa(K) \cup Ne(i)$ by contraction, which implies P2 by weak union.
\end{proof}

\begin{proof}[Proof of Theorem \ref{the:lp}]
The properties L1 and L2 imply respectively P1 and P2 by weak union. Note also that the pairwise Markov property implies the global property by Theorems \ref{the:zg}, \ref{the:zb} and \ref{the:pb}. Now, it is easy to verify that the local property holds. Specifically, the separations $i \ci_G Nd(i) \setminus K \setminus Fa(K) \setminus Mo(i) | Fa(K) \cup Mo(i)$ for all $i \in K$ with $K \in Cc(G)$ imply L1 by the global Markov property. Likewise, the separations $i \ci_G Nd(i) \setminus \{i\} \setminus Pa(i) \setminus Ne(i) \setminus Mo(Ne(i)) | Pa(i) \cup Ne(i) \cup Mo(Ne(i))$ for all $i \in K$ with $K \in Cc(G)$ imply L2 by the global Markov property.
\end{proof}

\begin{proof}[Proof of Theorem \ref{the:soundrules}]
Recall from the main text that $p(V \setminus X | do(X))$ satisfies the global Markov property with respect to $G_{V \setminus X}$. Then, it also satisfies the global Markov property with respect to $(G_{V \setminus X})'$, because this UCG is a supergraph of $G_{V \setminus X}$. Then, rule 1 holds. Actually, $G_{V \setminus X}$ and $(G_{V \setminus X})'$ represent the same separations over $V \setminus X$, because all the variables $F_{V_i}$ are observed (i.e., $F_{V_i}=a$ or $F_{V_i}=idle$) and, thus, they do not open new routes.

For the proof of rule 2, assume that $Z$ and $Y$ are singletons. The generalization to sets of variables is trivial. First, assume that $F_Z \dra Z$ is in $(G_{V \setminus X})'$. The antecedent of rule 2 implies that all the $W$-open routes in $(G_{V \setminus X})'$ from $Z$ to $Y$ start with an edge $Z \ra A$ or $Z \dra A$. Then, observing $Z=z$ and setting $Z=z$ produces the same effect on $Y$ and, thus, rule 2 holds. Second, assume that $F_Z \ra Z$ is in $(G_{V \setminus X})'$. The antecedent of rule 2 implies that all the $W$-open routes in $(G_{V \setminus X})'$ from $Z$ to $Y$ start with an edge $Z \ra B$ or $Z \dra B$ or with a subroute $Z - \cdots - Y$ or $Z - \cdots - A \ra B$ or $Z - \cdots - A \dra B$ or $Z - \cdots - A \dla B$. Then, observing $Z=z$ and setting $Z=z$ produces the same effect on $Y$ and, thus, rule 2 holds. The result may not be immediate when the route from $Z$ to $Y$ starts with a subroute $Z - \cdots - A \dla B$. The result follows from the fact that $A \in W$ and the interventional causal mechanism is $F_Z \ra Z$, i.e. the intervention interferes with $A$.

For the proof of rule 3 holds, assume that $Z$ and $Y$ are singletons. The generalization to sets of variables is trivial. First, assume that $F_Z \ra Z$ is in $(G_{V \setminus X})'$. The antecedent of rule 3 implies that all the $W$-open routes in $(G_{V \setminus X})'$ from $Z$ to $Y$ start with an edge $Z \la A$ or $Z \dla A$ or with a subroute $Z - B_1 - \cdots - B_n \la A$ such that $B_i \notin W$ for all $i$. Then, setting $Z=z$ has no effect on $Y$ and, thus, rule 3 holds. Second, assume that $F_Z \dra Z$ is in $(G_{V \setminus X})'$. The antecedent of rule 3 implies that all the $W$-open routes in $(G_{V \setminus X})'$ from $Z$ to $Y$ start with an edge $Z \la A$ or $Z \dla A$ or $Z - A$. Then, setting $Z=z$ has no effect on $Y$ and, thus, rule 3 holds. The result may not be immediate when the route from $Z$ to $Y$ starts with an edge $Z - A$. The result follows from the fact that the interventional causal mechanism is $F_Z \dra Z$, i.e. the intervention does not interfere with $A$.
\end{proof}

\begin{proof}[Proof of Corollary \ref{cor:lwf}]
Let $K_1, \ldots, K_n$ denote the topologically sorted chain components of $G$. Note that $K_i \setminus X \ci_{G_{V \setminus X}} K_1 \cup \cdots \cup K_{i-1} \setminus Pa(K_i) \setminus X | Pa(K_i) \setminus X$ for all $i$. Moreover, recall from the main text that $p(V \setminus X | do(X))$ satisfies the global Markov property with respect to $G_{V \setminus X}$. Then, we have that
\[
p(V \setminus X | do(X))= \prod_{K \in Cc(G)} p(K \setminus X |Pa(K) \setminus X, do(X)).
\]
Note also that $K \setminus X \ci_{(G_{V \setminus [ X \setminus Pa(K) ]})'} F_{Pa(K) \cap X} | Pa(K)$. Then, rule 2 implies that
\[
p(V \setminus X | do(X))= \prod_{K \in Cc(G)} p(K \setminus X |Pa(K), do(X \setminus Pa(K))).
\]
Note also that $K \setminus X \ci_{(G_{V \setminus [ X \setminus Pa(K) \setminus K ]})'} F_{K \cap X} | Pa(K) \cup (K \cap X)$ by the assumption that the interventional causal mechanism of each variable $A \in K \cap X$ is modeled as $F_A \ra A$, and the fact that $F_B$ is observed (i.e., $F_B=idle$) for each variable $B \in K \setminus X$. Then, rule 2 implies that
\[
p(V \setminus X | do(X))= \prod_{K \in Cc(G)} p(K \setminus X |Pa(K) \cup (K \cap X), do(X \setminus Pa(K) \setminus K)).
\]
Finally, note that $K \setminus X \ci_{G'} F_{X \setminus Pa(K) \setminus K} | Pa(K) \cup (K \cap X)$. Then, rule 3 implies that
\[
p(V \setminus X | do(X))= \prod_{K \in Cc(G)} p(K \setminus X |Pa(K) \cup (K \cap X)).
\]
\end{proof}

\begin{proof}[Proof of Lemma \ref{lem:rescaling}]
Let $\Lambda$ and $\overline{\Lambda}$ denote the error covariance matrices before and after the rescaling, respectively. Note that we only need to consider independences between singletons. In particular, $\epsilon_i$ is independent of $\epsilon_j$ conditioned on $\epsilon_L$ if and only if $(\Lambda_{ijL,ijL}^{-1})_{i,j}=0$ \citep[Proposition 5.2]{Lauritzen1996} and, thus, if and only if the $(j,i)$ minor of $\Lambda_{ijL,ijL}$ is zero, i.e. $det(M)=0$ where $M$ is the result of removing the row $j$ and column $i$ from $\Lambda_{ijL,ijL}$. Note that $det(M)=\sum_{\pi \in S} sign(\pi) \prod_k M_{k, \pi(k)}$ where $S$ denotes all the permutations over the number of rows or columns of $M$. Then, $det(\overline{M})=det(M) \prod_k 1/\lambda_k^2$ and, thus, $det(M)=0$ if and only $det(\overline{M})=0$.

A matrix is positive definite if and only if the determinants of all its upper-left submatrices are positive. Therefore, it follows from the previous paragraph that $\Lambda$ and $\overline{\Lambda}$ are both positive definite or none.
\end{proof}

\begin{proof}[Proof of Theorem \ref{the:main}]
Under the faithfulness assumption, we can identify the Markov equivalence class of $G$ from $p$ by, for instance, running the learning algorithm developed by \citet{Studeny1997b} for LWF CGs and by \citet{Penna2012} for AMP CGs. Assume to the contrary that there is a second CG $G'$ in the equivalence class that generates $p$ via Equation \ref{eq:recursion2c}. Let $N$ and $N'$ denote the nodes without parents in $G$ and $G'$. Assume that $N \neq N'$. Then, there exists some node $L \in V$ that is in $N$ but not in $N'$, or vice versa. Assume the latter without loss of generality. Then, there is an edge $Y \ra L$ that is in $G$ but not in $G'$. Let $Q$ denote all the parents of $L$ in $G$ except $Y$. Then, we have from $G$ that
\[
L=\beta_Q Q + \beta_Y Y + \epsilon_L
\]
by Equation \ref{eq:recursion2c}. Note that $\beta_Y \neq 0$ by the faithfulness assumption. Define $L^*=L|_{Q=q}$ and $Y^*=Y|_{Q=q}$ in distribution. Since $\epsilon_L \ci_p Q \cup \{Y\}$ by construction, we have from $G$ that
\[
L^*=\beta_Q q + \beta_Y Y^* + \epsilon_L
\]
in distribution \citep[Lemma 2]{Petersetal.2011} and, thus, $var(L^*) > var(\epsilon_L) = \lambda^2$. However, we have from $G'$ that $var(L^*) \leq \lambda^2$ \citep[Lemma A1]{PetersandBuhlmann2014}. This is a contradiction and, thus, $N=N'$. Note that the undirected edges between the nodes in $N$ must be the same in $G$ and $G'$, because Markov equivalent CGs have the same adjacencies as shown by \citet[Theorem 5.6]{Frydenberg1990b} for LWF CGs and by \citet[Theorem 5]{Anderssonetal.2001} for AMP CGs. For the same reason, the directed edges from $N$ to $V \setminus N$ must be the same in $G$ and $G'$. Since $G_{V \setminus N}$ and $G_{V \setminus N}'$ generate $p(V \setminus N |N=n)$ via Equation \ref{eq:recursion2c}, we can repeat the reasoning above replacing $G$, $G'$ and $p$ by $G_{V \setminus N}$, $G_{V \setminus N}'$ and $p(V \setminus N |N=n)$. This allows us to conclude that the nodes with no parents in $G_{V \setminus N}$ and their corresponding edges coincide with those in $G_{V \setminus N}'$. By continuing with this process, we can conclude that $G=G'$.
\end{proof}

\section*{Appendix B: Covariance Decomposition}\label{sec:appendixb}

It is known from path analysis in Gaussian acyclic directed mixed graphs that the covariance $\Sigma_{i,l}$ can be expressed as the sum for every open path between $l$ and $i$ of the product of path coefficients and error covariances for the edges in the path \citep{Wright1921,Pearl2009}. The path coefficient $\alpha_{B,A}$ corresponding to an edge $A \ra B$ represents the change in $B$ induced by raising $A$ one unit while keeping all the other variables constant. Note then that the path coefficients carry causal information. For directed and acyclic graphs, $\alpha_{B,A}$ is always identifiable and coincides with the partial regression coefficient $\beta_{B,A|Z(B,A)}$, where $Z(B,A)$ is a set of nodes that blocks all paths from $A$ to $B$ except $A \ra B$. For directed and acyclic graphs, we have then that
\[
\Sigma_{i,l}= \sum_{\rho \in \pi_{i,l}} \Sigma_{\rho_1,\rho_1} \prod_{n=2}^{|\rho|} \alpha_{\rho_n,\rho_{n-1}} = \sum_{\rho \in \pi_{i,l}} \Sigma_{\rho_1,\rho_1} \prod_{n=2}^{|\rho|} \beta_{\rho_n,\rho_{n-1} | Z(\rho_n,\rho_{n-1})}.
\]

Despite being symmetric, undirected edges represent interference by contagion and, thus, they have some features of causal relationships. Then, one may wonder whether it is possible to derive an expression similar to the one above for undirected graphs. \citet[p. 782]{JonesandWest2005} hint this possibility. The next lemma proves it formally: $\Sigma_{i,l}$ can indeed be expressed as the sum for every open path between $l$ and $i$ of the product of regression coefficients, error covariances and inflation factors for the edges in the path.

\begin{lemma}\label{lem:if}
Consider a Gaussian distribution over a set of random variables $K$. Let $\Sigma$ denote the covariance matrix of the distribution. Assume that the distribution satisfies the global Markov property with respect to an undirected graph $G$. Then,
\[
\Sigma_{i,l}= \sum_{\rho \in \pi_{i,l}} \Sigma_{\rho_1,\rho_1} \prod_{n=2}^{|\rho|} \beta_{\rho_n,\rho_{n-1} | K \setminus \{\rho_n,\rho_{n-1}\}} \frac{\Sigma_{\rho_n,\rho_n|\rho_1, \ldots, \rho_{n-1}}}{\Sigma_{\rho_n,\rho_n|K \setminus \{ \rho_n \}}}.
\]
\end{lemma}

\begin{proof}
Let $\overline{\rho}=K \setminus \rho$ for any $\rho \subseteq K$. Then,
\begin{align*}
\Sigma_{i,l}&= \sum_{\rho \in \pi_{i,l}} (-1)^{|\rho|+1} \frac{|\Omega_{\overline{\rho},\overline{\rho}}|}{|\Omega|} \prod_{n=2}^{|\rho|} \Omega_{\rho_{n-1}, \rho_n}\\
&= \sum_{\rho \in \pi_{i,l}} (-1)^{|\rho|+1} \frac{|\Omega_{\overline{\rho},\overline{\rho}}|}{|\Omega|} \prod_{n=2}^{|\rho|} \Omega_{\rho_{n-1}, \rho_n} \frac{\Omega_{\rho_n, \rho_n}}{\Omega_{\rho_n, \rho_n}}\\
&= \sum_{\rho \in \pi_{i,l}} \frac{|\Omega_{\overline{\rho},\overline{\rho}}|}{|\Omega|} \prod_{n=2}^{|\rho|} \beta_{\rho_n,\rho_{n-1} | K \setminus \{\rho_n,\rho_{n-1}\}} \Omega_{\rho_n, \rho_n}
\end{align*}
where the first equality is due to \citet[Theorem 1]{JonesandWest2005}, and the third is due to \citet[p. 130]{Lauritzen1996}. Then,
\begin{align*}
\Sigma_{i,l}&= \sum_{\rho \in \pi_{i,l}} \frac{|\Omega_{\overline{\rho},\overline{\rho}}|}{|\Omega_{\overline{\rho},\overline{\rho}}| |\Omega_{\rho,\rho|\overline{\rho}}|} \prod_{n=2}^{|\rho|} \beta_{\rho_n,\rho_{n-1} | K \setminus \{\rho_n,\rho_{n-1}\}} \Omega_{\rho_n, \rho_n}\\
&= \sum_{\rho \in \pi_{i,l}} |\Omega_{\rho,\rho|\overline{\rho}}^{-1}| \prod_{n=2}^{|\rho|} \beta_{\rho_n,\rho_{n-1} | K \setminus \{\rho_n,\rho_{n-1}\}} \Omega_{\rho_n, \rho_n}\\
&= \sum_{\rho \in \pi_{i,l}} |\Sigma_{\rho,\rho}| \prod_{n=2}^{|\rho|} \beta_{\rho_n,\rho_{n-1} | K \setminus \{\rho_n,\rho_{n-1}\}} \Omega_{\rho_n, \rho_n}
\end{align*}
by Schur's complement and inversion of a partitioned matrix. Then,
\begin{align*}
\Sigma_{i,l}&= \sum_{\rho \in \pi_{i,l}} |\Sigma_{\rho,\rho}| \prod_{n=2}^{|\rho|} \frac{\beta_{\rho_n,\rho_{n-1} | K \setminus \{\rho_n,\rho_{n-1}\}}}{\Sigma_{\rho_n,\rho_n|K \setminus \{ \rho_n \}}}\\
&= \sum_{\rho \in \pi_{i,l}} \Sigma_{\rho_1,\rho_1} \prod_{n=2}^{|\rho|} \beta_{\rho_n,\rho_{n-1} | K \setminus \{\rho_n,\rho_{n-1}\}} \frac{\Sigma_{\rho_n,\rho_n|\rho_1, \ldots, \rho_{n-1}}}{\Sigma_{\rho_n,\rho_n|K \setminus \{ \rho_n \}}}
\end{align*}
by \citet[pp. 129-130]{Lauritzen1996}.
\end{proof}

The variance ratio in the lemma is an inflation factor ($\geq 1$) that accounts for the overreduction of the variance of $\rho_n$ when conditioning on the rest of the variables in $K$. In other words, conditioning on the rest of the variables not only blocks all the rest of the paths from $\rho_{n-1}$ to $\rho_n$ but also overreduces the variance of $\rho_n$, which bias the causal effect of $\rho_{n-1}$ on $\rho_n$ represented by $\beta_{\rho_n,\rho_{n-1} | K \setminus \{\rho_n,\rho_{n-1}\}}$. The inflation factor remedies this. See \citet[Section 3]{Pearl2013} for some related phenomena (e.g., selection bias) in acyclic directed mixed graphs.

\bibliographystyle{plainnat}
\bibliography{UCGs22}

\end{document}